\documentclass[letterpaper]{article} 
\usepackage{aaai25}  
\usepackage{times}  
\usepackage{helvet}  
\usepackage{courier}  
\usepackage[hyphens]{url}  
\usepackage{graphicx} 
\usepackage{natbib}  
\usepackage{caption} 
\usepackage{algorithm}
\usepackage{algorithmic}

\usepackage{newfloat}
\usepackage{listings}
\DeclareCaptionStyle{ruled}{labelfont=normalfont,labelsep=colon,strut=off} 
\floatstyle{ruled}
\newfloat{listing}{tb}{lst}{}
\floatname{listing}{Listing}
\title{Recursive Aggregates as Intensional Functions in Answer Set Programming:\\Semantics and Strong Equivalence}
\author{Jorge Fandinno,  Zachary Hansen}
\affiliations{
    \textsuperscript{\rm 1}University of Nebraska Omaha, Omaha, NE, USA\\
    \{jfandinno,zachhansen\}@unomaha.edu
}

\usepackage{bibentry}

\usepackage{amsmath}
\usepackage{amsthm}
\usepackage{amsfonts}
\usepackage{amssymb}
\usepackage{listings}
\usepackage{xspace}
\usepackage{xcolor}
\usepackage{trimclip}
\newcommand{\I}{\mathcal{I}}

\renewcommand{\P}{\mathcal{P}}
\renewcommand{\S}{\mathcal{S}}
\newcommand{\F}{\mathcal{F}}
\newcommand{\less}{\prec^{\P\!\F}}
\newcommand{\lesseq}{\preceq^{\P\!\F}}
\newcommand{\lessP}[1]{\prec^{#1}}
\newcommand{\lesseqP}[1]{\preceq^{#1}}

\newcommand{\modelsht}{\models_{ht}}

\newcommand{\sneg}{\raisebox{1pt}{\rotatebox[origin=c]{180}{\ensuremath{\neg}}}}

\def\rar{\rightarrow}
\def\lrar{\leftrightarrow}

\def\ba{\begin{array}}
\def\ea{\end{array}}
\def\bce{\begin{center}}
\def\ece{\end{center}}

\def\beq{\begin{equation}}
\def\eeq#1{\label{#1}\end{equation}}

\newcommand{\cli}{\mathit{cli}}
\newcommand{\taug}{\tau^{\cli}}
\newcommand{\taud}{\tau^{\mathit{dlv}}}
\newcommand{\setsg}{\sets^{\cli}}
\newcommand{\setsd}{\sets^{\mathit{dlv}}}

\newtheorem{definition}{Definition}

\newtheorem{problem}{Problem}

\newtheorem{theorem}{Theorem}

\newtheorem{proposition}[theorem]{Proposition}
\newtheorem{lemma}[theorem]{Lemma}

\newcounter{myCounter}

\newcounter{problemCounter}
\newcommand\p[1]{\smallskip\noindent\textbf{Problem} \arabic{problemCounter} [#1].\stepcounter{problemCounter}}

\def\bc{\begin{choices}}
\def\ec{\end{choices}}
\def\boc{\begin{oneparchoices}}
\def\eoc{\end{oneparchoices}}

\def\p{\part}
\def\bq{\begin{questions}}
\def\eq{\end{questions}}
\def\bp{\begin{parts}}
\def\ep{\end{parts}}

\def\beq{\begin{equation}}
\def\eeq#1{\label{#1}\end{equation}}
\def\ba{\begin{array}}
\def\ea{\end{array}}

\def\rar{\rightarrow}

\DeclareMathOperator{\myinf}{\mathit{inf}}

\DeclareMathOperator{\mysup}{\mathit{sup}}

\DeclareMathOperator{\head}{\mathit{Head}}

\newcommand{\anthem}{\textsc{anthem}}
\newcommand{\core}{\texttt{ASP\nobreakdash-Core\nobreakdash-2}}
\newcommand{\clingo}{\texttt{clingo}}
\newcommand{\dlv}{\texttt{dlv}}

\newcommand{\boldX}{\mathbf{X}}
\newcommand{\boldx}{\mathbf{x}}
\newcommand{\boldy}{\mathbf{y}}
\newcommand{\boldY}{\mathbf{Y}}
\newcommand{\boldZ}{\mathbf{Z}}
\newcommand{\boldz}{\mathbf{z}}

\newcommand{\boldt}{\mathbf{t}}
\newcommand{\boldu}{\mathbf{u}}

\newcommand{\boldd}{\mathbf{d}}
\newcommand{\boldc}{\mathbf{c}}

\newcommand{\boldl}{\mathbf{l}}

\newcommand{\sets}{\mathit{s}}
\newcommand{\tuple}[1]{\langle#1\rangle}

\newcommand{\ruleo}{\;\hbox{:-}\;}

\newcommand{\fsum}{\mathit{sum}}
\newcommand{\fcount}{\mathit{count}}

\newcommand{\pmemF}[2]{#1 \in #2}

\newcommand{\ftuplen}[1]{\mathit{tuple}_{#1}}
\newcommand{\ftuple}{\ftuplen{}}

\newcommand{\tset}{t_{\mathit{set}}}

\newcommand{\dtuple}{d_{\mathit{tuple}}}

\newcommand{\sorttuple}{s_{\mathit{tuple}}}

\newcommand{\sortset}{s_{\mathit{set}}}
\newcommand{\sortsuper}{s_{\mathit{gen}}}

\newcommand{\vartuple}{T}

\newcommand{\eqdef}{%
  \mathrel{\vbox{\offinterlineskip\ialign{%
    \hfil##\hfil\cr%
    $\scriptscriptstyle\mathrm{def}$\cr%
    \noalign{\kern1pt}%
    $=$\cr%
    \noalign{\kern-0.1pt}%
}}}}

\newcommand{\Ans}[1]{\mathcal{#1}}

\newcommand{\domaintuple}{d_{\mathit{tuple}}}

\newcommand{\finterp}{standard interpretation}

\newcommand{\htinterp}{ht\nobreakdash-interpretation\xspace}
\newcommand{\htinterps}{ht\nobreakdash-interpretations\xspace}

\newcommand{\agginterp}{agg\nobreakdash-interpretation\xspace}
\newcommand{\agginterps}{agg\nobreakdash-interpretations\xspace}

\newcommand{\htmodel}{ht\nobreakdash-model\xspace}

\newcommand{\agghtinterp}{agg\nobreakdash-\htinterp}
\newcommand{\agghtinterps}{agg\nobreakdash-\htinterps}

\newcommand{\aggmodel}{agg\nobreakdash-model\xspace}
\newcommand{\aggmodels}{agg\nobreakdash-models\xspace}

\newcommand{\agghtmodel}{agg\nobreakdash-ht\nobreakdash-model\xspace}
\newcommand{\agghtmodels}{agg\nobreakdash-ht\nobreakdash-models\xspace}

\newcommand{\aggstable}{agg\nobreakdash-stable\xspace}

\newcommand{\numeral}[1]{\overline{#1}}

\DeclareMathOperator{\Not}{\mathit{not}}
\newcommand{\HF}{\Gamma}
\newcommand{\Int}{I}
\newcommand{\grp}[3]{\mathit{gr}^{#1}_{#2}(#3)}
\newcommand{\gr}[2]{\grp{\P\!\F}{#1}{#2}}

\makeatletter
\newcommand\thefontsize[1]{{#1 The current font size is: \f@size pt\par}}
\makeatother

\newcommand{\g}[2]{\mathit{gr}_{#1}(#2)}

\def\lrar{\leftrightarrow}
\def\beq{\begin{equation}}
\def\eeq#1{\label{#1}\end{equation}}
\def\ba{\begin{array}}
\def\ea{\end{array}}

\def\anthem{{\sc anthem}}

\def\p2f{\hbox{p2f}}

\def\head{\emph{Head\/}}

\newcommand{\FLP}{\mathit{FLP}}

\newcommand{\NN}[1]{\mathit{n}(#1)}
\newcommand{\PPP}[1]{\mathit{p}(#1)}
\newcommand{\PNN}[1]{\mathit{pn}(#1)}

\newcommand{\HT}{\ensuremath{\mathit{HT}}}
\newcommand{\AGG}{\ensuremath{\mathit{AGG}}}

\begin{document}

\maketitle

\begin{abstract}
    This paper shows that the semantics of programs with aggregates implemented by the solvers \clingo\ and \dlv\ can be characterized as extended First-Order formulas with intensional functions in the logic of Here\nobreakdash-and-There.
    Furthermore, this characterization can be used to study the strong equivalence of programs with aggregates under either semantics.
    We also present a transformation that reduces the task of checking strong equivalence to reasoning in classical First\nobreakdash-Order logic, which serves as a foundation for automating this procedure.
\end{abstract}

\section{Introduction}\label{sec:introduction}

Answer set programming~(ASP) is a declarative programming paradigm well-suited for solving knowledge-intensive search and optimization problems~\cite{lifschitz19a}.
Its success relies on the combination of a rich knowledge representation language with effective solvers.
Some of its most useful constructs are~\emph{aggregates}, that is, functions that apply to sets.
The semantics of aggregates have been extensively studied in the literature~\citep{siniso02a,doporo03a,pedebr07a,sonpon07b,ferraris11a,fapfle11a,gelzha14a,gelzha19a,cafascsc17b}.
In most cases, they rely on the idea of \emph{grounding}---a process that replaces all variables by variable-free terms.
%
%
This makes reasoning about First-Order (FO) programs with aggregates cumbersome and it does not allow the use of classical FO theorem provers for verifying properties about this class of programs.

Though several approaches describe the semantics of aggregates bypassing the need for grounding, most of these approaches only allow a restricted class of aggregates~\cite{lee08a,lifschitz22a} or use some extension of the logical language~\cite{baleme11b,leme12a,aschzhzh15a,cafafape18a}.
Recently, \citet{fahali22a,fahali24a} showed how to translate logic programs with aggregates into FO sentences, which, after the application of the SM operator~\cite{feleli11a}, captures the \core\ semantics.
%
%
Though most practical problems can be represented within the restrictions of the \core\ semantics, some notable exceptions are more naturally represented using recursive aggregates, which are not allowed by \core.
One of these examples is the \emph{Company Control problem}, which consists of finding companies that control other companies by
(directly or indirectly)
owning a majority of their shares.
This problem has been encoded in the literature using the following logic program~\cite{pedebr07a,fapfle11a,mupira90a,kemstu91a,kroysa97a}:
\begin{flalign}
&\texttt{\small ctrStk(C1,C1,C2,P)} \ruleo \texttt{\small ownsStk(C1,C2,P)}.&
\\
&\texttt{\small ctrStk(C1,C2,C3,P)} \ruleo \texttt{\small controls(C1,C2),}&
\notag
\\
&\phantom{\texttt{\small ctrStk(C1,C2,C3,P)} \ruleo}\texttt{\small ownsStk(C2,C3,P)}.
\\
&\texttt{\small controls(C1,C3)} \ruleo \texttt{\small company(C1), company(C3),}&
\notag
\\
&\hspace*{32pt}\texttt{\small \#sum\{P,C2\,:\,ctrStk(C1,C2,C3,P)\}\,>\,50}.&
\label{eq:rule.control}
\end{flalign}%
where atom~$\texttt{ownsStk(C1,C2,P)}$ means that company~$\texttt{C1}$ directly owns~$\texttt{P\%}$ of the shares of company~$\texttt{C2}$;
$\texttt{ctrStk(C1,C2,C3,P)}$ means that company~$\texttt{C1}$ controls~$\texttt{P\%}$ of the shares of company~$\texttt{C3}$ through company~$\texttt{C2}$ that it controls; and~$\texttt{controls(C1,C3)}$ means that company~$\texttt{C1}$ controls company~$\texttt{C3}$.
Another area where allowing recursive aggregates is important is in the study of \emph{strong equivalence}~\cite{lipeva01a,lipeva07a}.
The strong equivalence problem consists of determining whether two programs have the same behavior in any context.
Even if the programs we are analyzing do not contain recursion, adding some context may introduce it.

In this paper, we show that the translation introduced by~\citeauthor{fahali22a} can also be used for programs with recursive aggregates if we interpret functions in an intensional way~\cite{linwan08a,cabalar11a,lifschitz12a,balduccini13a,barlee19a}.
We focus on the Abstract Gringo~\cite{gehakalisc15a} generalization of the semantics by~\citet{ferraris11a}, which is used in the answer set solver~\texttt{clingo}, and the semantics by~\citet{fapfle11a}, which are used in the answer set solver~\texttt{dlv}.
We prove that the translation introduced by~\citeauthor{fahali22a} coincides with the Abstract Gringo semantics when we interpret the function symbols representing sets according to the semantics for intensional functions by~\citeauthor{barlee19a}.
For~\texttt{dlv}, we introduce a similar translation, which uses a second form of negation.
We show how we can use these translations to express the strong equivalence of the two programs and how to reduce this problem to reasoning in classical FO logic.
%


\section{Preliminaries}\label{sec:preliminaries}
We start by reviewing the syntax of programs with aggregates and presenting an extension of the logic of Quantified Here-and-There~\cite{peaval08a} with intensional functions that is suited for programs with aggregates.

\paragraph{Syntax of programs with aggregates.}
We follow here the presentation by~\citet{fahali22a}.
We assume a {\em (program) signature} with three countably infinite sets of symbols:
\emph{numerals}, \emph{symbolic constants} and \emph{program variables}.
We also assume a 1-to-1 correspondence between numerals and integers; the numeral corresponding to an integer~$n$ is denoted by~$\numeral{n}$.
\emph{Program terms} are either numerals, symbolic constants, variables or one of the special symbols~$\myinf$ and~$\mysup$.
A program term (or any other expression) is~\emph{ground} if it contains no variables.
We assume that a total order on ground terms is chosen such that
\begin{itemize}
	\item \strut $\myinf$ is its least element and $\mysup$ is its greatest element,
	\item for any integers~$m$ and $n$, $\numeral{m} < \numeral{n}$ iff $m < n$, and
	\item for any integer~$n$ and any symbolic constant~$c$, $\numeral{n} < c$.
\end{itemize}
An \emph{atom} is an expression of the form~$p(\boldt)$, where $p$~is a symbolic constant and~$\boldt$ is a list of program terms.
A \emph{comparison} is an expression of the form~$t \prec t'$, where $t$ and $t'$ are program terms and~$\prec$ is one of the  \emph{comparison symbols}:%
\begin{align}
	\label{rel:1}
	= \quad \neq \quad < \quad > \quad \leq \quad \geq
\end{align}
An \emph{atomic formula} is either an atom or a comparison.
A \emph{basic literal} is an atomic formula possibly preceded by one or two occurrences of \emph{not}.
An {\em aggregate element} has the form
\begin{gather}
    t_1,\dots,t_k : l_1,\dots,l_m
    \label{eq:agel:2}
\end{gather}
where each $t_i$ ($1\leq i\leq k$) is a program term and each $l_i$ ($1\leq i\leq m$) is a basic literal.
%
%
%
%
%
An {\em aggregate atom} is of form
    ${\texttt{\#op}\{E\} \prec u}$
where~$\texttt{op}$ is an operation name, $E$ is an aggregate element,
$\prec$ is one of the comparison symbols in~(1), and $u$ is a program term, called \emph{guard}.
We consider operation names~$\texttt{count}$ and~$\texttt{sum}$.
For example, the expression
$$
\texttt{\#sum\{P,C2\,:\,ctrStk(C1,C2,C3,P)\}\,>\,50}
$$
in the body of rule~\eqref{eq:rule.control} is an aggregate atom.
An {\em aggregate literal} is an aggregate atom possibly preceded by one or two occurrences of~{\em not}.
A \emph{literal} is either a basic literal or an aggregate literal.
A \emph{rule} is an expression of the form
\begin{align}
	\label{rule:4}
	\head \ruleo B_1, \dots, B_n,
\end{align}
where
$\head$ is an atom or symbol~$\bot$, and each $B_i$ is a literal.
%

We call the symbol $\hbox{:-}$ the {\em rule operator}. We call the left-hand side of the rule operator the {\em head}, the right-hand side of the rule operator the {\em body}.
%
%
When the head of the rule is an atom we call the rule \emph{normal},
and when it is the symbol~$\bot$ we call it a \emph{constraint}.
When the body of a normal rule is empty, we call the rule a \emph{fact}.
%
A {\em program} is a set of rules.

Each operation name~$\texttt{op}$ is associated with a function~$\widehat{\texttt{op}}$  that maps every set of tuples of ground terms to a ground term.
%
%
%
If the first member of a tuple~$\boldt$ is a numeral~$\numeral{n}$ then we say that integer~$n$ is the weight of~$\boldt$, otherwise the weight of~$\boldt$ is~$0$.
For any set~$\Delta$ of tuples of ground terms,
\begin{itemize}
    \item $\widehat{\texttt{count}}(\Delta)$
    is the numeral corresponding to the cardinality of~$\Delta$,
    if~$\Delta$ is finite; and $\mysup$ otherwise.%

    \item $\widehat{\texttt{sum}}(\Delta)$
    is the numeral corresponding to the sum of the weights of all tuples in~$\Delta$,
    if~$\Delta$ contains finitely many tuples with non-zero weights; and $0$ otherwise.%
    %
    If~$\Delta$ is empty, then~${\widehat{\texttt{sum}}(\Delta) = 0}$.
\end{itemize}
Though we illustrate the semantics of aggregates using the operation names~$\texttt{count}$ and~$\texttt{sum}$, the semantics can be extended to other operation names by adding the appropriate functions~$\widehat{\texttt{op}}$~\cite{fahali24a}.

\paragraph{Many-sorted logic and extended FO formulas.}
A many-sorted signature consists of symbols of three
kinds---\emph{sorts}, \emph{function constants}, and
\emph{predicate constants}.
A reflexive and transitive \emph{subsort}
relation is defined on the set of sorts.
A tuple $s_1,\dots,s_n$ ($n\geq 0$) of \emph{argument sorts} is assigned
to every function constant and to every predicate constant; in addition, a
\emph{value sort} is assigned to every function constant.
Function constants with $n=0$ are called \emph{object constants}.
For every sort, an infinite sequence of \emph{object variables} of that sort is chosen.
\emph{Terms} and \emph{atomic formulas} over a (many-sorted) signature~$\sigma$ are defined as usual with the consideration that the sort of a term must be a subsort of the sort of the function or predicate constant of which it is an argument.
\emph{Extended First\nobreakdash-Order formulas} over~$\sigma$ are formed from atomic formulas
and the 0-place connective~$\bot$ (falsity) using the unary connective~$\sneg$, the binary
connectives $\wedge$, $\vee$, $\to$ and the quantifiers $\forall$, $\exists$.
We define the usual abbreviations:
$\neg F$ stands for $F\to\bot$ and
$F\lrar G$ stands for $(F\to G)\wedge (G\to F)$.
%
We have two negation symbols \mbox{($\neg$ and~$\sneg$)} and both correspond to classical negation in the context of classical FO logic.
The symbol~$\neg$ represents standard negation in the logic of \mbox{Here-and-There} and corresponds to default negation in logic programs under the \clingo\ semantics, while symbol~$\sneg$ is a new connective and it represents default negation under the \dlv\ semantics.
%
%
\emph{Interpretations}, \emph{sentences}, \emph{theories}, \emph{satisfaction} and~\emph{models} are defined as usual with the additional condition that~${I \models \sneg F}$ iff~$I \not\models F$.
A \emph{standard} FO formula (resp. sentence, theory) is a formula (resp. sentence, theory) without the new operator~$\sneg$.

\paragraph{Stable Model Semantics with Intensional Functions.}
%
Let~$I$ and~$H$ be two interpretations of a signature~$\sigma$ and~$\P$ and~$\F$ respectively be sets of predicate and function constants of~$\sigma$.
We write~$H \lesseq I$ if
\begin{itemize}
    \item $H$ and~$I$ have the same universe for each sort;
    \item ${p^H \subseteq p^I}$ for every predicate constant~$p$ in~$\P$ and\\${p^H = p^I}$ for every predicate constant~$p$ not in~$\P$; and
    \item $f^H = f^I$ for every function constant~$f$ not in~$\F$.
\end{itemize}
%
%
If~$I$ is an interpretation of a signature~$\sigma$ then
by~$\sigma^I$ we denote the signature
obtained from~$\sigma$ by adding, for every element $d$ of a domain $|I|^s$,
its \emph{name} $d^*$ as an object constant of sort~$s$.
An \emph{\htinterp} of~$\sigma$ is a pair
$\langle H,I\rangle$, where~$H$ and~$I$ are interpretations
of~$\sigma$ such that~${H \lesseq I}$.
(In terms of many\nobreakdashes-sorted Kripke models, $I$ is the there-world, and~$H$
is the here-world).
The satisfaction relation~$\modelsht$ between an
HT\nobreakdash-interpretation $\langle H, I\rangle$ of~$\sigma$
and a sentence~$F$ over~$\sigma^I$ is defined recursively as follows:
\begin{itemize}
\item
  $\langle H, I\rangle \modelsht p(\boldt)$,
  if $I \models p(\boldt)$ and $H \models p(\boldt)$;

\item
$\langle H, I\rangle \modelsht t_1=t_2$ if $t_1^I=t_2^I$ and~$t_1^H=t_2^H$;

\item $\tuple{H,I} \modelsht \sneg F$ if both~$I \not\models F$ and~$H \not\models F$;


\item
$\langle H, I\rangle \modelsht F\land G$ if
$\langle H, I\rangle \modelsht F$ and
$\langle H, I\rangle \modelsht G$;
\item
$\langle H, I\rangle \modelsht F\lor G$ if
$\langle H, I\rangle \modelsht F$ or
$\langle H, I\rangle \modelsht G$;
\item
  $\langle H, I\rangle \modelsht F\to G$ if
  $I \models F\to G$,
    and\\\hspace*{87pt}
  $\langle H, I\rangle \not\modelsht F$ or $\langle H, I\rangle \modelsht G$;
\item
  $\langle H, I\rangle\modelsht\forall X\,F(X)$
 if $\langle H, I\rangle\modelsht F(d^*)$
  \\
  for each~$d\in|I|^s$, where~$s$ is the sort of~$X$;
\item
  $\langle H, I\rangle\modelsht\exists X\,F(X)$
 if $\langle H, I\rangle\modelsht F(d^*)$
  \\
  for some~$d\in|I|^s$, where~$s$ is the sort of~$X$.

\end{itemize}
If $\langle H, I\rangle \modelsht F$ holds, we say that $\langle H, I\rangle$ \emph{ht-satisfies}~$F$ and that $\langle H, I\rangle$ is an \emph{\htmodel} of~$F$.
If it is clear from the context that the $\modelsht$ entailment relation is referred to, we will simply say that $\langle H, I\rangle$ \emph{satisfies}~$F$.
We say that~$\tuple{H,I}$ \emph{satisfies}
a set of sentences~$\Gamma$ if it satisfies every sentence~$F$ in~$\Gamma$.
%

We write~${H \less I}$ if~${H \lesseq I}$ and~${H \neq I}$.
A model~$I$ of a set~$\Gamma$ of sentences is called \emph{stable} if there is no~${H \less I}$ such that~$\tuple{H,I}$ satisfies~$\Gamma$.
For finite standard theories, this definition of stable models coincides with the definition of one by~\citet{barlee19a} when sets~$\P$ and~$\F$ respectively contain the intensional predicate and function constants.
For (possibly infinite) standard theories with~$\F = \emptyset$, each stable model~$I$ corresponds to the equilibrium 
model~$\tuple{I,I}$ by~\citet{peaval08a}.

\section{Logic Programs With Aggregates as Extended Many-Sorted First-Order Sentences}\label{sec:translations}

We present here translations~$\taug$ and~$\taud$ that turn a program into extended FO sentences with equality over a signature~$\sigma(\P,\S)$ of {\em three sorts}; $\P$ and~$\S$ are sets of predicate and \emph{set symbols}, respectively.
Superscripts~$\mathit{cli}$ and~$\mathit{dlv}$ refer to the semantics of~\clingo\ and \dlv, respectively.
\paragraph{Target Signature.}
A \emph{set symbol} is a pair~$E/\boldX$, where $E$~is an aggregate element and~$\boldX$ is a list of variables occurring in~$E$.
For brevity's sake, each set symbol~$E/\boldX$ is assigned a short name~$|E/\boldX|$.
The target signature is of three sorts.
The first sort is called the \emph{general sort} (denoted~$\sortsuper$);  all program terms are of this sort.
The second sort is called the {\em tuple sort} (denoted~$\sorttuple$); it contains entities that are \emph{tuples} of objects of the general sort.
The third sort is called the {\em set sort} (denoted~$\sortset$); it contains entities that are \emph{sets} of elements of the second sort, that is, sets of tuples of objects of the general sort.
%
%
Signature~$\sigma(\P,\S)$ contains:
\begin{enumerate}
\item all ground terms as object constants of the general sort;

\item all predicate symbols in~$\P$ with all arguments of the general sort;

\item\label{en:1:3} comparison symbols other than equality as binary predicate symbols whose arguments are of the general sort;

\item predicate constant~$\in\!\!/2$ with the first argument of the sort tuple and the second argument of the sort set;

\item function constant~$\ftuple/k$ with arguments of the general sort and value of the tuple sort for each set symbol~$E/\boldX$ in~$\S$ with~$E$ of the form of~\eqref{eq:agel:2};

\item\label{en:2} unary function constants~$\fcount$ and~$\fsum$ whose argument is of the set sort and whose value is of the general sort;

\item\label{en:3:5} for each set symbol~$E/\boldX$ in~$\S$ where~$n$ is the number of variables in~$\boldX$, function constants~$\setsg_{|E/\boldX|}$ and~$\setsd_{|E/\boldX|}$
with~$n$ arguments of the general sort and whose value is of the set sort.




\end{enumerate}
We assume that~$\P$ is the set of intensional predicates and that the set of intensional functions is the set of all function symbols corresponding to set symbols~in~$\S$.
We use infix notation in constructing atoms that utilize predicate symbols of comparisons (${>,\geq,<,\leq,\neq}$) and the set membership predicate~$\in$.
%
%
Function constants~$\setsg_{|E/\boldX|}$ and~$\setsd_{|E/\boldX|}$ are used to represent sets occurring in aggregates for the \clingo\ and \dlv\ semantics, respectively.
Each of these function constants maps an $n$\nobreakdash-tuple of ground terms~$\boldx$ to the set of tuples represented%
\footnote{For a tuple $\boldX$ of distinct variables, a tuple~$\boldx$ of ground terms of the same length as~$\boldX$, and an expression~$\alpha$, by~$\alpha^\boldX_\boldx$ we denote the expression obtained from~$\alpha$ by substituting~$\boldx$ for~$\boldX$.}
by~$E^\boldX_\boldx$.
%
%
These claims are formalized below.
%
%
%

About a predicate symbol~$p/n$, we say that it \emph{occurs} in a program~$\Pi$ if there is an atom of the form~$p(t_1,\dotsc,t_n)$ in~$\Pi$.
For set symbols, we need to introduce first the concepts of global variables and set symbols.
A variable is said to be \emph{global} in a rule if
\begin{enumerate}
    \item it occurs in any non-aggregate literal, or
    \item it occurs in a guard of any aggregate literal.
\end{enumerate}
We say that set symbol~$E/\boldX$ occurs in rule~$R$ if this rule contains an aggregate literal with the aggregate element~$E$ and~$\boldX$ is the lexicographically ordered list of all variables in~$E$ that are global in~$R$.
We say that~$E/\boldX$ occurs in a program~$\Pi$ if~$E/\boldX$ occurs in some rule of the program.
For instance, in rule~\eqref{eq:rule.control}
the global variables are~\texttt{C1} and~\texttt{C3}.
Set symbol~$E_{ctr}/\boldX_{ctr}$  occurs in this rule where~$E_{ctr}$ stands for the aggregate element~$\texttt{P,C2\,:\,ctrStk(C1,C2,C3,P)}$ and~$\boldX_{ctr}$ is the list of variables~$\texttt{C1},\texttt{C3}$.
We denote by~$\setsg_{ctr}/2$ and~$\setsd_{ctr}/2$ the function symbols associated with this set symbol for the \clingo\ and \dlv\ semantics, respectively.

When discussing a program~$\Pi$, we assume a signature~$\sigma(\P,\S)$ such that~$\P$ and~$\S$ are the sets that contain all predicate symbols and all set symbols occurring in~$\Pi$, respectively.
%
Furthermore, when it is clear from the context, we write just~$\sigma$ instead of~$\sigma(\P,\S)$.

\paragraph{Translations.}
We now describe translations
that convert a program into a set of extended FO sentences.
We use~$\tau^x_\boldZ$ and~$\tau^x$ to denote the rules that are common to both translations when~$x$ is replaced by either~$\cli$ or~$\mathit{dlv}$.
Given a list~$\boldZ$  of global variables in some rule $R$,
we define~$\taug_\boldZ$ and~$\taud_\boldZ$ for all elements of $R$ as follows.
\begin{enumerate}
    \item
    for every atomic formula~$A$ occurring outside of an aggregate literal, its translation~$\tau^x_\boldZ A$ is~$A$ itself; $\tau^x_\boldZ \bot$ is~$\bot$;
    \item
    for an aggregate atom~$A$ of form~$\texttt{\#count}\{E\} \prec u$ or~$\texttt{\texttt{\#sum}}\{E\} \prec u$, its translation~$\tau^x_\boldZ$ is the atom
    \begin{gather*}
        \fcount(\sets^x_{|E/\boldX|}(\boldX)) \prec u
        \text{\normalsize\ or }
        \fsum(\sets^x_{|E/\boldX|}(\boldX)) \prec u
        \end{gather*}
respectively, where~$\boldX$ is the lexicographically ordered list of the variables in~$\boldZ$ occurring in~$E$;

    \item for every (basic or aggregate) literal of the form~$\Not A$
its translation~$\taug_\boldZ (\Not A)$ is~$\neg \taug_\boldZ A$
and
its translation~$\taud_\boldZ (\Not A)$ is~$\sneg \taug_\boldZ A$;
for every literal of the form~$\Not\Not  A$
its translation~$\taug_\boldZ (\Not\Not A)$ is~$\neg\neg  \taug_\boldZ A$
and
its translation~$\taud_\boldZ (\Not\Not A)$ is~$\sneg\sneg  \taud_\boldZ A$.
%
\label{item:tauY*.last}
\end{enumerate}
We now define the translation~$\tau^x$ as follows:
\begin{enumerate}
    \setcounter{enumi}{3}
    \item for every rule~$R$ of form (4), its translation~$\tau^x R$ is the universal closure of
    $$
    \tau^x_\boldZ B_1 \wedge \dots \wedge \tau^x_\boldZ B_n \to \tau^x_\boldZ \head,
    $$
    where~$\boldZ$ is the list of the global variables of~$R$.
    \item for every program~$\Pi$, its translation~$\tau^x \Pi$ is the theory containing~$\tau^x R$ for each rule~$R$ in~$\Pi$.
\end{enumerate}
$\taug$ and~$\taud$ only differ in the translation of negation and the use of different function constants for set symbols.
For instance, rule~\eqref{eq:rule.control} is translated into the universal closure of
\begin{gather}
    \begin{aligned}
        \mathit{company}(C_1) \wedge \mathit{company}(C_3)
        &
        \\
        \wedge\, \fsum(\sets^x_{ctr}(C_1,C_3)) > 50
        &
        \to \mathit{controls}(C_1,C_3)
    \end{aligned}
    \label{eq:rule.control.translated}
\end{gather}
where variables~$C_1$ and~$C_3$ are of the general sort, and~$x$ is either~$\cli$ or~$\mathit{dlv}$ depending on the semantics considered.


\paragraph{Standard interpretations.}

A \emph{\finterp~$I$} is an interpretation of~$\sigma(\P,\S)$  that satisfies the following \emph{conditions}:
\begin{enumerate}
    \item universe $|I|^{\sortsuper}$ is the set containing all ground terms of the general sort;

    \item \label{interp.domain.tuple} universe $|I|^{\sorttuple}$ is the set of all tuples of form $\tuple{d_1,\dotsc,d_k}$ with~$d_i\in |I|^{\sortsuper}$ for each set symbol~$E/\boldX$ in~$\S$ with~$E$ of the form of~\eqref{eq:agel:2};

    \item every element of~$|I|^{\sortset}$ is a subset of~$|I|^{\sorttuple}$;




    \item $I$ interprets each ground program term as itself;

    \item $I$ interprets predicate symbols $>,\geq,<,\leq$ according to the total order chosen earlier;


    \item $I$ interprets each tuple term of form~$\ftuple(t_1,\dotsc,t_k)$ as the tuple~$\tuple{t_1^I,\dotsc,t_k^I}$;

    \item $\in^I$ is the set of pairs~$(t,s)$ s.t. tuple~$t$ belongs to set~$s$;





    \item for term $\tset$ of sort $\sortset$, $\fcount(\tset)^I$ is $\widehat{\texttt{count}}(\tset^I)$;

    \item for term $\tset$ of sort $\sortset$, $\fsum(\tset)^I$ is $\widehat{\texttt{sum}}(\tset^I)$;
\end{enumerate}
An \emph{\agginterp} is a standard interpretation~$I$ satisfying, for every set symbol~$E/\boldX$ in~$\S$ with~$E$ of the form of~\eqref{eq:agel:2} and for all~$x \in \{ \cli, \mathit{dlv} \}$, that~$\sets_{|E/\boldX|}^x(\boldx)^I$ is the set of all tuples of the form $\tuple{(t_1)^{\boldX\boldY}_{\boldx\boldy},\dots,(t_k)^{\boldX\boldY}_{\boldx\boldy}}$ such that $I$ satisfies~$\tau^x({l_1})^{\boldX\boldY}_{\boldx\boldy} \wedge \dots \wedge \tau^x({l_m})^{\boldX\boldY}_{\boldx\boldy}$.

For instance, the program representing the Company Control problem has a unique set symbol that is associated with the function symbols~$\sets^x_{ctr}/2$ (${x \in \{ \cli, \mathit{dlv}\}}$).
If~$I$ is an \agginterp such that $\mathit{ctrStk}^I$ is the set containing~$(c_1,c_2,c_3,10)$ and~$(c_1,c_4,c_3,20)$, it follows that~$\mathit{\sets^x_{ctr}(c_1,c_3)}^I$ (with~$x \in \{ \cli, \mathit{dlv}\}$) is the set containing tuples~$\tuple{10,c_2}$ and~$\tuple{20,c_4}$.

An \htinterp~$\tuple{H,I}$ is said to be \emph{standard} if both~$H$ and~$I$ are standard.
An \emph{\agghtinterp} is a standard~\htinterp~$\tuple{H,I}$ satisfying that~$I$ is an \agginterp and the following conditions for every set symbol~$E/\boldX$ in~$\S$ with~$E$ of the form of~\eqref{eq:agel:2} :
\begin{itemize}

    \item $\setsg_{|E/\boldX|}(\boldx)^H$ is the set of all tuples of form\\$\tuple{(t_1)^{\boldX\boldY}_{\boldx\boldy},\dots,(t_k)^{\boldX\boldY}_{\boldx\boldy}}$ such that~$\tuple{H,I}$ satisfies\\ $\taug({l_1})^{\boldX\boldY}_{\boldx\boldy} \wedge \dots \wedge \taug({l_m})^{\boldX\boldY}_{\boldx\boldy}$; and

    \item $\setsd_{|E/\boldX|}(\boldx)^H$ is the set of all tuples of form\\$\tuple{(t_1)^{\boldX\boldY}_{\boldx\boldy},\dots,(t_k)^{\boldX\boldY}_{\boldx\boldy}}$ such that~$H$ satisfies\\$\taud({l_1})^{\boldX\boldY}_{\boldx\boldy} \wedge \dots \wedge \taud({l_m})^{\boldX\boldY}_{\boldx\boldy}$.
\end{itemize}
where~$\boldY$ is the lexicographically ordered list of the variables occurring in~$E$ that are not in~$\boldX$.
Let us consider now an \agghtinterp~$\tuple{H,I}$ where~$I$ is as described above and~$\mathit{ctrStk}^H$ is the set containing~$(c_1,c_2,c_3,10)$.
Then, $\mathit{\sets^x_{ctr}(c_1,c_3)}^H$ is the set containing tuples~$\tuple{10,c_2}$.
In this example, there is no difference between the semantics of \clingo\ and \dlv.
As an example of where these semantics differ, consider an \agghtinterp~$\tuple{H,I}$ with~${p^H = r^H = \emptyset}$ and~${p^I = q^I = q^H = r^I= \{1\}}$, and~rule
\begin{gather}
    \texttt{\small p(1)\::-\:\#sum\{X\::\:q(X),\:not\:r(X)\}\:<\:1.}
    \label{eq:clingo.dlv.difference.rule}
\end{gather}
This rule is translated into the sentences
\begin{align}
    \fsum(\setsg_{qr}) < 1
    &\to p(1)
    \label{eq:difference.clingo.semantics}
    \\
    \fsum(\setsd_{qr}) < 1
    &\to p(1)
    \label{eq:difference.dlv.semantics}
\end{align}
for the \clingo\ and \dlv\ semantics, respectively.
It is clear that~$I$ satisfies both rules because~$1$ belongs to~$p^I$.
However, when considering the \agghtinterp~$\tuple{H,I}$, only the second rule is satisfied.
On the one hand, $(\sets^x_{qr})^I$ (with~$x \in \{\cli, \mathit{dlv}\}$) is the empty set.
Furthermore,  $(\setsg_{qr})^H$ is also the empty set because~${\tuple{H,I} \not\models \neg r(1)}$, and the antecedent of~\eqref{eq:difference.clingo.semantics} is satisfied.
Then, the rule is not satisfied because the consequent is not satisfied due to~$1$ not belonging to~$p^H$.
On the other hand, $(\setsd_{qr})^H$ is the set containing~$1$ because~${H \models q(1) \wedge \sneg r(1)}$.
Hence, $\tuple{H,I}$ does not satisfy the antecedent of~\eqref{eq:difference.dlv.semantics} and the rule is satisfied.
%



We now define stable models for programs with aggregates.
A model of a formula or theory that is also an \agginterp is called an~\emph{\aggmodel} and an \agghtinterp that satisfies a formula or theory is called an \emph{\agghtmodel}.

\begin{definition}
    \label{def:agg.stable.model}
    An \aggmodel $I$ of~$\Gamma$ is an \emph{\aggstable\ model} of~$\Gamma$ if there is no \agghtmodel~$\tuple{H,I}$ with~${H \less  I}$.
\end{definition}

\section{Correspondence with \clingo\ and \dlv}\label{sec:relation-with-semantics}

We establish now the correspondence between the semantics of programs with aggregates introduced in the previous section and the semantics of the solver~\clingo, named Abstract Gringo~\citep{gehakalisc15a}, and the solver~\dlv, which is based on the FLP\nobreakdash-reduct~\cite{fapfle11a}.
These semantics are stated in terms of infinitary formulas following the work by~\citet{harlif19a}.


\vspace{3pt}
\noindent{\bf Infinitary Formulas.}
We extend the definitions of infinitary logic~\cite{truszczynski12a} to formulas with intensional functions and the~$\sneg$ connective.
For every nonnegative integer~$r$,
\emph{infinitary ground formulas of rank~$r$} are
defined recursively:
\begin{itemize}
\item every ground atom in $\sigma$ is a formula of rank~0,
\item if $\HF$ is a set of formulas, and~$r$ is the smallest nonnegative integer that is greater than the ranks of all elements of $\HF$,
then $\HF^\land$ and $\HF^\lor$ are formulas of rank~$r$,
\item if $F$ and $G$ are formulas, and~$r$ is the smallest nonnegative
integer that is greater than the ranks of~$F$ and~$G$, then $F\rar G$ is a
formula of rank~$r$,
\item if $F$ is a formula, and~$r$ is the smallest nonnegative
integer that is greater than the rank $F$, then $\sneg F$ is a formula of rank~$r$.
\end{itemize}
We write $\{F,G\}^\land$ as $F\land G$,
$\{F,G\}^\lor$ as $F\lor G$, and $\emptyset^{\lor}$ as~$\bot$.

We extend the satisfaction relation for \htinterps\ to infinitary formulas by adding the following two conditions to the definition for FO formulas:
\begin{itemize}
    \item ${\tuple{H,I} \modelsht\HF^\land}$ if for every formula $F$ in~$\HF$,
    ${\tuple{H,I} \modelsht F}$,
    \item ${\tuple{H,I} \modelsht\HF^\lor}$ if there is a formula $F$ in~$\HF$
    such that ${\tuple{H,I} \modelsht F}$,
\end{itemize}
We write~$I \models F$ if~$\tuple{I,I} \modelsht F$.

\citet{truszczynski12a} defines the satisfaction of infinitary formulas with respect to sets of ground atoms instead of FO interpretations.
Such a satisfaction relation for infinitary formulas can be defined when we have no intensional functions.
An infinitary ground formula is \emph{propositional} if it does not contain intensional functions.
For a signature~$\sigma$, by~$\sigma^p$ we denote the set of all ground atoms over~$\sigma$ that do not contain intensional functions.
\renewcommand{\I}{\mathcal{A}}
Subsets of a propositional signature $\sigma^p$ are called \emph{propositional interpretations}.
The satisfaction relation between a propositional interpretation~$\I$ and an infinitary propositional formula is defined recursively:
\begin{itemize}
\item for every ground atom $A$ from $\sigma$, $\I\models A$ if $A$ belongs to~$\I$,
\item $\I\models\HF^\land$ if for every formula $F$ in~$\HF$,
$\I\models F$,
\item $\I\models\HF^\lor$ if there is a formula $F$ in~$\HF$
such that $\I\models F$,
\item $\I\models F\rar G$ if $\I\not\models F$ or $\I\models G$,
\item $\I\models \sneg F$ if $\I \not\models F$.
\end{itemize}
In the following, if~$I$ is an interpretation, then~$\mathcal{I}$ denotes the set of atomic formulas of~$\sigma^p$ satisfied by~$I$.
With this notation, the following result is easily proved by induction.

\begin{proposition}\label{prop:infinitary.interpretations.cl}
    Let~$F$ be an infinitary propositional formula.
    Then, $I \models F$ iff~$\mathcal{I} \models F$.
\end{proposition}%

\paragraph{Grounding.}
The \emph{grounding} of a FO sentence allows us to replace quantifiers with infinitary conjunctions and disjunctions.
Formally, the \emph{grounding of a First\nobreakdash-Order sentence~$F$ with respect to an interpretation~$\Int$ and sets~$\P$ and~$\F$ of intensional predicate and function symbols} is defined as follows:
\begin{itemize}
\item $\gr{\Int}{\bot} \text{ is } \bot$;

\item $\gr{\Int}{p(\boldt)} \text{ is } p(\boldt)$ if~$p(\boldt)$ contains intensional symbols;

\item $\gr{\Int}{p(\boldt)} \text{ is } \top$
if $p(\boldt)$ does not contain intensional symbols and~${I \models p(\boldt)}$;
and ${\gr{\Int}{p(\boldt)} \text{ is } \bot}$ otherwise;

\item $\gr{\Int}{t_1 = t_2} \text{ is } (t_1 = t_2)$ if $t_1$ or~$t_2$ contain intensional symbols;

\item $\gr{\Int}{t_1 = t_2} \text{ is } \top$ if $t_1$ and~$t_2$ do not contain intensional symbols and $t_1^I = t_2^I$ and $\bot$ otherwise;

\item $\gr{\Int}{\sneg F} \text{ is } \sneg \gr{\Int}{F}$;

\item $\gr{\Int}{F \otimes G} \text{ is } \gr{\Int}{F} \otimes \gr{\Int}{G}$ if $\otimes$ is $\wedge$, $\vee$, or $\to$;
\item $\gr{\Int}{\exists X \, F(X)} \text{ is } \{ \gr{\Int}{F(u)} \mid u \in |\Int|^{s} \}^{\vee}$ if $X$ is a variable of sort~$s$;
\item $\gr{\Int}{\forall X \, F(X)} \text{ is } \{ \gr{\Int}{F(u)} \mid u \in |\Int|^{s} \}^{\wedge}$ if $X$ is a variable of sort~$s$.
\end{itemize}
For a first\nobreakdash-order theory~$\Gamma$,
we define~$\gr{\Int}{\Gamma} = \{ \gr{I}{F} \mid F \in \Gamma \}^\wedge$.
For any first\nobreakdash-order theory~$\Gamma$, $\gr{I}{\Gamma}$ is an  infinitary formula, which may contain intensional functions or the $\sneg$ connective.
We write~$\g I\cdot$ instead of~$\gr I\cdot$ when it is clear from the context.

\begin{proposition}\label{lem:grounding.ht}
    $\langle H,I\rangle \modelsht F$
    iff $\langle H,I\rangle \modelsht \g IF$.
\end{proposition}

\paragraph{Standard formulas and minimal models.}
We say that an infinitary propositional formula is \emph{standard} if it does not contain the~$\sneg$ connective.
The definitions of the semantics of \clingo\ and \dlv\ only use standard infinitary formulas and rely on the notion of minimal models.
A propositional interpretation~$\I$ satisfies a set $\HF$ of formulas, in symbols~$\I \models \Gamma$, if it satisfies every formula in~$\HF$.
We say that a set~$\mathcal{A}$ of atoms is a $\subseteq$\nobreakdash-minimal model of a set of infinitary formulas~$\Gamma$,
if $\mathcal{A} \models \Gamma$ and there is no~$\mathcal{B}$  satisfying~${\mathcal{B} \models \Gamma}$ and~${\mathcal{B} \subset \mathcal{A}}$.

\paragraph{Clingo.}
The \emph{FT\nobreakdash-reduct} $F^\I$ of a standard infinitary  formula~$F$ with respect to a propositional interpretation~$\I$ is defined recursively.
If $\I\not\models F$ then $F^\I$ is~$\bot$; otherwise,
\begin{itemize}
\item for every ground atom~$A$, $A^\I$ is $A$
\item $(\HF^\land)^\I = \{G^\I\ |\ G\in\HF\}^\land$,
\item $(\HF^\lor)^\I  = \{G^\I\ |\ G\in\HF\}^\lor$,
\item $(G\rar H)^\I$ is $G^\I\rar H^\I$,
\end{itemize}
We say that a propositional interpretation~$\mathcal{A}$ is an \emph{FT\nobreakdash-stable model} of a formula~$F$ if it is a $\subseteq$\nobreakdash-minimal model of~$F^\mathcal{A}$.
We say that a set~$\mathcal{A}$ of ground atoms is a \emph{clingo answer set} of a program~$\Pi$ if~$\mathcal{A}$ is an FT\nobreakdash-stable model of~$\tau\Pi$ where~$\tau$ is the translation from logic programs to infinitary formulas defined by~\citet{gehakalisc15a}.
The following result states that the usual relation between \htinterps and the FT\nobreakdash-reduct is satisfied in our settings.

\begin{proposition}\label{prop:infinitary.interpretations}
    %
    Let~$F$ be a standard infinitary formula of~$\sigma^p$.
    Then, $\tuple{H,I} \modelsht F$ iff $\mathcal{H} \models F^{\mathcal{I}}$.
\end{proposition}%


For \agghtinterps we can state the relation~$\less$ in terms of the atomic formulas satisfied by it as follows:

\begin{proposition}
    \label{prop:less.ht}
    Let~$\tuple{H,I}$ be an \agghtinterp.
    Then, $H \less I$ iff $\mathcal{H} \subset \mathcal{I}$.
\end{proposition}

Using Propositions~\ref{prop:infinitary.interpretations.cl}\nobreakdash-\ref{prop:less.ht}, we can prove the relation between clingo answer sets and \aggstable models of the corresponding FO theory.
Note that clingo answer sets are propositional interpretations while \aggstable models of FO theories are FO interpretations.
To fill this gap, we introduce the following notation.
%
If~$I$ is an agg\nobreakdash-stable model of~$\taug\Pi$, we say that $\Ans{I}$ is a \emph{fo\nobreakdash-clingo answer set} of~$\Pi$.

\begin{theorem}
    \label{thm:abstrac.gringo.correspondence}
    The fo\nobreakdash-clingo answer sets of any program coincide with its clingo answer sets.
\end{theorem}

\begin{proof}[Proof sketch]
    The core of the proof consists of showing that~${\tuple{H,I} \modelsht \taug\Pi}$ iff~${\mathcal{H} \models (\tau\Pi)^{\mathcal{I}}}$ holds.
    By Proposition~\ref{lem:grounding.ht}, we get ${\langle H,I\rangle \modelsht \taug\Pi}$
    iff ${\langle H,I\rangle \modelsht \g I {\taug\Pi}}$.
    Note that~$\g I {\taug\Pi}$ is not an infinitary propositional formula, because it may contain intensional functions.
    Thus, we cannot apply Proposition~\ref{prop:infinitary.interpretations} directly.
    However, we can prove that $\langle H,I\rangle \modelsht \g I {\taug\Pi}$ iff ${\langle H,I\rangle \modelsht \tau\Pi}$ holds and use Proposition~\ref{prop:infinitary.interpretations} to prove the stated result.
    Finally, Proposition~\ref{prop:less.ht} is used to state the correspondence between stable models of~$\tau\Pi$ and \aggstable models of~$\taug\Pi$.
\end{proof}

\paragraph{The \dlv\ semantics.}

Similarly to the Abstract Gringo semantics, the \dlv\ semantics can be stated in terms of the same translation~$\tau$ to infinitary formulas, but using a different reduct~\cite{harlif19a}.
%
%
%
Let $F$ be an implication~${F_1 \to F_2}$.
Then, the \emph{FLP-reduct} $\FLP(F, \I)$ of~$F$ w.r.t. a propositional interpretation $\I$
is $F$ if $\I \models F_1$, and $\top$ otherwise.
For a conjunction of implications~$\mathcal{F}^\wedge$,
we define
$$\FLP(\mathcal{F}^\wedge, \I) \ = \ \{ \FLP(F,\I) \mid F \in \mathcal{F}\}^\wedge$$
%
A set~$\I$ of ground atoms is an \emph{FLP-stable model} of $F$ if it is a $\subseteq$\nobreakdash-minimal model of $\FLP(F, \I)$.
We say that a set~$\mathcal{A}$ of ground atoms is a \emph{dlv answer set} of a program~$\Pi$ if~$\mathcal{A}$ is an FLP\nobreakdash-stable model of~$\tau\Pi$.
If~$I$ is an agg\nobreakdash-stable model of~$\taud\Pi$, we say that $\Ans{I}$ is a \emph{fo\nobreakdash-dlv answer set} of~$\Pi$.

\begin{theorem}
    \label{thm:dlv.correspondence}
    The fo\nobreakdash-dlv answer sets of any program coincide with its dlv answer sets.
\end{theorem}

\begin{proof}[Proof sketch]
    The structure of the proof is analogous to the one of Theorem~\ref{thm:abstrac.gringo.correspondence}.
    Here, the key step of the proof consists of showing~${\tuple{H,I} \modelsht \taud\Pi}$ iff both~${\Ans{I} \models \tau\Pi}$ and~${\mathcal{H} \models \FLP(\tau\Pi,\Ans{I})}$.
    %
\end{proof}







\section{Strong Equivalence}\label{sec:strong.equivalece}

We say that two programs~$\Pi_1$ and $\Pi_2$
are {\em strongly equivalent for \clingo} if program~$\Pi_1 \cup \Delta$ and $\Pi_2 \cup \Delta$ have the same \clingo\ answer sets for any program~$\Delta$.

We assume a signature~$\sigma(\P,\S)$ where~$\P$ and~$\S$  are the sets that respectively contain all predicate and all set symbols occurring in~$\Pi_1 \cup \Pi_2$.

\begin{theorem}\label{thm:strong.equivalece.clingo}
    The following conditions are equivalent:
    \begin{itemize}
        \item $\Pi_1$ and $\Pi_2$ are strongly equivalent for \clingo;
        \item $\taug(\Pi_1)$ and $\taug(\Pi_2)$ have the same \agghtmodels.
    \end{itemize}
\end{theorem}
Let us consider the program formed by rule~\eqref{eq:clingo.dlv.difference.rule}.
This program has~$\{ p(1) \}$ as its unique clingo answer set.
Similarly, the program formed by the rule
\small
\begin{gather}
    \texttt{p(1)\::-\:not\:\#sum\{X\::\:q(X),\,not\:r(X)\}\,>=\,1.}
    \label{eq:clingo.dlv.difference.rule2}
\end{gather}
\normalsize
also has~$\{ p(1) \}$ as its unique clingo answer set.
However, these two programs are not strongly equivalent under the \clingo\ semantics.
To illustrate this claim, consider an \agghtinterp~$\tuple{H,I}$ with~${p^H =  r^H = r^I = \emptyset}$, and~$p^I = q^H = \{1\}$, and~${q^I = \{1, -1\}}$.
On the one hand, $\tuple{H,I}$ satisfies~\eqref{eq:difference.clingo.semantics} because its antecedent is not satisfied as we have~$\fsum(\setsg_{qr})^H = 1$.
On the other hand, $\tuple{H,I}$ does satisfy the formula
\begin{align}
    \neg\fsum(\setsg_{qr}) \geq 1
    &\to p(1)
    \label{eq:difference.clingo.semantics2}
\end{align}
obtained by applying~$\taug$ to~\eqref{eq:clingo.dlv.difference.rule2}.
Note that in the scope of negation, we only look at the value in~$I$ and we have~${\fsum(\setsg_{qr})^I = 0}$.
By Theorem~\ref{thm:strong.equivalece.clingo}, this implies that the two programs are not strongly equivalent under the \clingo\ semantics.
This assertion can be confirmed by adding context
\begin{align}
    \texttt{q(1). \quad q(-X)\::-\:p(X). \quad :-\:not\:p(1).}
    \label{eq:context}
\end{align}
When added to rule~\eqref{eq:clingo.dlv.difference.rule}, the resulting program has no clingo answer sets, but when added to rule~\eqref{eq:clingo.dlv.difference.rule2}, the resulting program has~$\{ q(1),\,q(-1),\,p(1) \}$ as its unique answer set.

Similarly,
we say that two programs~$\Pi_1$ and~$\Pi_2$
are {\em strongly equivalent for \dlv} if programs~$\Pi_1 \cup \Delta$ and $\Pi_2 \cup \Delta$ have the same \dlv\ answer sets for any program~$\Delta$.

\begin{theorem}\label{thm:strong.equivalece.dlv}
    The following conditions are equivalent.
    \begin{itemize}
        \item $\Pi_1$ and $\Pi_2$ are strongly equivalent for \dlv;
        \item $\taud(\Pi_1)$ and $\taud(\Pi_2)$ have the same \agghtmodels.
    \end{itemize}
\end{theorem}
Though programs containing rules~\eqref{eq:clingo.dlv.difference.rule} and~\eqref{eq:clingo.dlv.difference.rule2} are not strongly equivalent for \clingo, they are strongly equivalent for \dlv.
Applying~$\taud$ to rule~\eqref{eq:clingo.dlv.difference.rule2} yields formula
\begin{align}
    \sneg\fsum(\setsd_{qr}) \geq  1
    &\to p(1)
    \label{eq:difference.dlv.semantics1}
\end{align}
and any \agghtinterp~$\tuple{H,I}$ satisfies the antecedent of~\eqref{eq:difference.dlv.semantics1}
iff~${H \not\models \fsum(\setsd_{qr}) \geq 1}$
and~${I \not\models \fsum(\setsd_{qr}) \geq 1}$
iff~${H \models \fsum(\setsd_{qr}) < 1}$
and~${I \models \fsum(\setsd_{qr}) < 1}$
iff~${\tuple{H,I}}$ satisfies~$\fsum(\setsd_{qr}) < 1$, which is the antecedent of~\eqref{eq:difference.dlv.semantics}.
By Theorem~\ref{thm:strong.equivalece.dlv}, this implies that the two programs are strongly equivalent for \dlv.

We can also use these translations to establish strong equivalence results across the two semantics.

\begin{theorem}\label{thm:strong.equivalece.croos}
    The following conditions are equivalent.
    \begin{itemize}
        \item $\taug(\Pi_1)$ and $\taud(\Pi_2)$ have the same \agghtmodels.
        \item the clingo answer sets of~$\Pi_1 \cup \Delta$ coincide with the dlv answer sets of~$\Pi_2 \cup \Delta$ for every set~$\Delta$ of rules such that~$\taug(\Delta)$ and $\taud(\Delta)$ have the same \agghtmodels.
    \end{itemize}
\end{theorem}
\noindent
In particular, note that~$\taug(\Delta)$ and $\taud(\Delta)$ are equivalent for any set of rules without aggregates nor double negation.
The restriction on the \aggmodels of~$\taug(\Delta)$ and $\taud(\Delta)$ is necessary because the same added rules may have different behavior under the two semantics.
As an example, consider program~$\Pi_1$ formed by rule~\eqref{eq:clingo.dlv.difference.rule} plus constraint
\begin{align}
    \texttt{:-\:q(X),\:r(X).}
    \label{eq:context2}
\end{align}
and program~$\Pi_2$ formed by rule~\eqref{eq:clingo.dlv.difference.rule2} plus constraint~\eqref{eq:context2}.
Recall that~$\taug\eqref{eq:clingo.dlv.difference.rule}$ is sentence~$\eqref{eq:difference.clingo.semantics}$ and~$\taud\eqref{eq:clingo.dlv.difference.rule2}$ is sentence~$\eqref{eq:difference.dlv.semantics1}$.
As we discussed above, any \agghtinterp satisfies the antecedent of~\eqref{eq:difference.dlv.semantics1} iff it satisfies~${\fsum(\setsd_{qr}) < }1$.
Hence, it is enough to show~${(\setsg_{qr})^H = (\setsd_{qr})^H}$ for every \agghtinterp~$\tuple{H,I}$ that satisfies~${\forall X \neg(q(X) \wedge r(X))}$.
Every such~$\tuple{H,I}$, satisfies that~${H \models q(c)}$ implies~${H \not\models r(c)}$ for every object constant~$c$ of general sort.
Hence~$\tuple{H,I}$ satisfies~$q(c) \wedge \neg r(c)$ iff~$H$ satisfies~$q(c) \wedge \sneg r(c)$ for every object constant~$c$ of general sort.
This implies~${(\setsg_{qr})^H = (\setsd_{qr})^H}$.
It is well\nobreakdash-known that for programs where all aggregate literals are positive, the \clingo\ and \dlv\ semantics coincide~\cite{ferraris11a,harlif19a}.
As illustrated by this example, Theorem~\ref{thm:strong.equivalece.croos} enables us to prove this correspondence for some programs with non-positive aggregates.\footnote{Recall that an aggregate literal is called \emph{positive} if it is not in the scope of negation and negation does not occur within its scope~\cite{harlif19a}.}


\section{Strong Equivalence using Classical Logic}\label{sec:strong.equivalece.classical}

In this section, we show how an additional
syntactic transformation~$\gamma$ allows us to replace the logic of Here\nobreakdash-and-There by classical FO theory.
This also allows us to remove the non\nobreakdash-standard negation~$\sneg$ and replace the semantic condition that characterizes \agginterps in favor of axiom schemata.
This is a generalization of the transformation by~\citet{fanlif23a} and it is similar to the one used by~\citet{barlee19a} to define the SM operator for FO formulas with intensional functions.

We define a signature~$\hat\sigma$ that is obtained from the signature~$\sigma$ by adding, for every predicate symbol~$p$ other than comparison symbols~\eqref{rel:1}, a new predicate symbol~$\hat{p}$ of the same arity and sorts; and for every function symbol~$\sets^x_{|E/\boldX|}$ with~$x \in \{ \mathit{cli}, \mathit{dlv} \}$, a new function symbol~$\hat{\sets}^x_{|E/\boldX|}$.



For any expression~$E$ of signature~$\sigma$, by~$\hat{E}$ we denote
the expression of~$\hat\sigma$
obtained from~$E$ by replacing every occurrence of every
predicate symbol~$p$ by $\hat{p}$ and every occurrence of function symbol~$\sets^x_{|E/\boldX|}$ by~$\hat{\sets}^x_{|E/\boldX|}$.
The translation~$\gamma$, which relates the logic of here-and-there
to classical logic, maps formulas over~$\sigma$ to formulas
over~$\hat{\sigma}$.
It is defined recursively:
\begin{itemize}
\item $\gamma F=F \wedge \hat{F}$ if~$F$ is atomic,
\item $\gamma(\neg F)=\neg \hat{F}$,
\item $\gamma(\sneg F)=\neg \hat{F} \wedge \neg \NN{F} $,
\item $\gamma(F \otimes G)=\gamma F \otimes \gamma G$ with~$\otimes \in \{ \wedge, \vee \}$.
\item $\gamma(F\to G)=(\gamma F \to \gamma G)\land (\hat{F}\to \hat{G})$,
\item $\gamma(\forall X\,F)=\forall X\,\gamma F$,
\item $\gamma(\exists X\,F)=\exists X\,\gamma F$.
\end{itemize}
where~$\NN{F}$ is the result of replacing all occurrences of~$\sneg$ by~$\neg$ in~$F$.
To apply~$\gamma$ to a set of formulas means to apply~$\gamma$ to each of its members.
Note that~$\gamma \Gamma$ is always a standard FO theory (without the~$\sneg$ connective) over the signature~$\hat\sigma$.

For any \htinterp~$\tuple{H,I}$ of~$\sigma$,
$I^H$ stands for the interpretation of~$\hat\sigma$ that has the same domain as~$I$, interprets symbols not in~$\P \cup \F$ in the same way as~$I$, interprets the other function symbol as
${f^{I^H} =f^H}$ and~${\hat{f}^{I^H} =f^I}$, and other
predicate constants as follows:
\begin{align*}
    I^H\models p({\bf d}^*) &\hbox{ iff } H \models p({\bf d}^*);
    &\hspace*{5pt}
    I^H\models \hat{p}({\bf d}^*)&\hbox{ iff }I\models p({\bf d}^*).
\end{align*}

\begin{proposition}\label{prop:prima.transformation.ht}
    ${\tuple{H,I} \modelsht \Gamma}$ iff ${I^H\models \gamma \Gamma}$.
\end{proposition}

The following set of formulas characterizes which interpretations of the signature~$\hat\sigma$ correspond to \htinterps.
By~$\HT$ we denote the set of all formulas of the form~${\forall\boldX (p(\boldX) \to \hat{p}(\boldX))}$ for every predicate symbol~$p \in \P$.
%
%
By $\AGG$ we denote the set of all sentences of the form
\begin{align}
\forall \boldX \vartuple \big(\pmemF{\vartuple}{\hat{\sets}^x_{|E/\boldX|}(\boldX)}
        &\leftrightarrow
        \exists \boldY \hat{F}^{\mathit{x}}\big)
\label{eq:agg_element.clingo.there}
\\
\forall \boldX \vartuple \big(\pmemF{\vartuple}{\setsg_{|E/\boldX|}(\boldX)}
        &\leftrightarrow
        \exists \boldY \gamma (F^{\mathit{cli}})\big)
\label{eq:agg_element.clingo.here}
\\
\forall \boldX \vartuple \big(\pmemF{\vartuple}{\setsd_{|E/\boldX|}(\boldX)}
&\leftrightarrow
\exists \boldY \NN{F^{\mathit{dlv}}}\big)
\label{eq:agg_element.dlv.here}
\end{align}
for every~$E/\boldX$ in~$\S$ with~$E$ of the form of~\eqref{eq:agel:2} and where
\begin{align*}
F^{\mathit{cli}} \ &\text{ is } \vartuple = \ftuple(t_1,\dots,t_k) \wedge  \taug(l_1)\wedge\cdots\wedge \taug(l_m)
\\
F^{\mathit{dlv}} \ &\text{ is } \vartuple = \ftuple(t_1,\dots,t_k) \wedge  \taud(l_1)\wedge\cdots\wedge \taud(l_m)
\end{align*}

\begin{proposition}\label{prop:prima.agg.interpretation}
    An interpretation of the signature~$\hat\sigma$ satisfies~$\HT$ and~$\AGG$ iff it can be represented in the form $I^H$ for some \agghtinterp~$\tuple{H,I}$.
\end{proposition}

We are ready to state the main result of this section showing that we can use classical FO logic to reason about strong equivalence under the \clingo\ and \dlv\ semantics.

\begin{theorem}\label{thm:strong.equivalece.classical}
    Finite programs~$\Pi_1$ and~$\Pi_2$ are strongly equivalent under the \clingo\ semantics iff all standard interpretations of~$\hat\sigma$ satisfy the sentence
    \begin{align*}
        \bigwedge \HT \wedge \bigwedge \AGG \to (F_1 \leftrightarrow F_2)
    \end{align*}
    where~$F_i$ is the conjunction of all sentences in~$\gamma\taug\Pi_i$.
    The same holds if we replace~\clingo\ and~$\taug$ by~\dlv\ and~$\taud$.
\end{theorem}

\section{Discussion and Conclusions}

In this paper, we provided a characterization of the semantics of logic programs with aggregates which bypasses grounding.
We focus on the semantics for recursive aggregates used by ASP solvers \clingo\ and \dlv.
Our characterization reflects the intuition that aggregates are functions that apply to sets, usually missing in most formal characterizations of aggregates, which treat them as monolithic constructs.
To achieve that, we translate logic programs with aggregates into First-Order sentences with intensional functions, establishing a connection between these two extensions of logic programs.
We also show how this characterization can be used to study the strong equivalence of programs with aggregates and variables under either semantics.
Finally, we show how to reduce the task of checking strong equivalence to reasoning in classical First-Order logic, which serves as a foundation for automating this procedure.
We also axiomatize the meaning of the symbols used to represent sets.
The axiomatization of the symbols representing aggregate operations ($\mathtt{sum}$ and~$\mathtt{count}$) developed by~\citet{fahali22a} for non\nobreakdash-recursive aggregates also applies to recursive aggregates because these function symbols stay non\nobreakdash-intensional.
Immediate future work includes the integration of this characterization of aggregates with the formalization of arithmetics used by the verification tool \anthem~\cite{falilusc20a,fahalilite23a} and the implementation of a new verification tool that can accommodate programs with aggregates.


\section*{Acknowledgements}
This research is partially supported by NSF CAREER award 2338635.
Any opinions, findings, and conclusions or recommendations expressed in this material are those of the authors and do not necessarily reflect the views of the National Science Foundation.

\bibliography{krr.bib, procs.bib}

\section{Proof of Results}
\appendix
\subsection{Some Results on Here-and-There Logic}\label{sec:results_ht}

The following results show that some of the usual properties of the logic of Here-and-There are preserved in the extension introduced here.
\begin{proposition}\label{prop:persistence}
The following properties hold:
\begin{itemize}
  \item $\tuple{I,I} \modelsht F$ iff~$I \models F$.
  \item If~$\tuple{H,I} \modelsht F$, then~$I \models F$.
  \item $\tuple{H,I} \modelsht \neg F$ iff~$I \not\models F$.
  \item $\tuple{H,I} \modelsht \neg\neg F$ iff~$I \models F$.
\end{itemize}
\end{proposition}

\begin{proof}
  \emph{Item~1} is immmediate when~$H=I$. \emph{Item~2.} If~$F$ is an atomic sentence or a sentence of the forms~$\sneg F_1$ or~$F_1 \to F_2$, the result follows from the definition of~$\modelsht$.
  The remaining cases are proved by induction on the size of~$F$.
  \emph{Item~3.} $\tuple{H,I} \modelsht \neg F$ iff~$\tuple{H,I} \modelsht F \to\bot$ iff~$I \models F \to \bot$ and $\tuple{H,I} \not\modelsht F$ iff~$I \not\models F$ (the last equivalece is a consequence of Item~2).
  \emph{Item~4} is an immmediate consequence of Item~3.
\end{proof}

The second item of Proposition~\ref{prop:persistence} shows that the persistence property of the logic of Here-and-There is preserved in the extension introduced here.

The following result sheds some light on the behavior of the new negation connective.

\begin{proposition}
  \label{prop:ht:facts:abbr:neg}
  The following properties hold:
  \begin{itemize}
  \item $\tuple{H,I} \modelsht \sneg F$ iff~$I \models \sneg F$ and~$H \models \sneg F$.
  \item $\tuple{H,I} \modelsht \sneg\sneg F$ iff~$I \models F$ and~$H \models F$.
  \item $\tuple{H,I} \modelsht \sneg\sneg F$ implies~$\tuple{H,I} \not\modelsht \sneg F$.

  \item $\tuple{H,I} \modelsht \sneg\sneg p(\boldt)$ iff~$\tuple{H,I} \models p(\boldt)$.
  \end{itemize}
\end{proposition}

\begin{proof}
\noindent\emph{Item~1}:
$\tuple{H,I} \modelsht \sneg F$
iff $I\not\models F$ and $H \not\models F$ (by definition)
iff $I\models \sneg F$ and $H \models \sneg F$ (by definition).
\\[3pt]
\noindent\emph{Item~2}:
$\tuple{H,I} \modelsht \sneg\sneg F$
iff $I \not\models \sneg F$ and $H \not\models \sneg F$ (by definition)
iff $I \models F$ and $H \models F$.
\\[3pt]
\noindent\emph{Item~3}:
$\tuple{H,I} \modelsht \sneg\sneg F$ implies $I \models F$ and $H \models F$ (Item~2) and, thus, $\tuple{H,I} \not\modelsht \sneg F$.
\\[3pt]
\noindent\emph{Item~4}:
${\tuple{H,I} \modelsht \sneg\sneg p(\boldt)}$
iff~$I \models p(\boldt)$ and~$H \models p(\boldt)$ (by Item~2)
iff~$\tuple{H,I} \models p(\boldt)$ (by definition).
\end{proof}


\begin{proposition}\label{prop:ht.facts}
  The following properties hold if~$\boldt$ does not contain intensional symbols.
  \begin{itemize}
  \item $\tuple{H,I} \modelsht p(\boldt)$ iff~$H \models p(\boldt)$,

  \item $\tuple{H,I} \modelsht \sneg p(\boldt)$ iff~$\tuple{H,I} \modelsht \neg p(\boldt)$ iff~$I \not\models p(\boldt)$.

\end{itemize}
\end{proposition}

\begin{proof}
  \noindent\emph{Item~1}: ${\tuple{H,I} \modelsht p(\boldt)}$
  iff ${H \models p(\boldt)}$ and ${I \models p(\boldt)}$
  iff ${\boldt^H \in p^H}$ and ${\boldt^I \in p^I}$
  iff ${\boldt^H \in p^H}$ (because ${p^H \subseteq p^I}$ and~${\boldt^H = \boldt^I}$)
  iff ${H \models p(\boldt)}$.
  \\[3pt]
  \noindent\emph{Item~2}: $\tuple{H,I} \modelsht \sneg p(\boldt)$
  iff $H \not\models p(\boldt)$ and $I \not\models p(\boldt)$
  iff $\boldt^H \notin p^H$ and $\boldt^I \notin p^I$
  iff $\boldt^I \notin p^I$ ($p^H \subseteq p^I$ and~$\boldt^H = \boldt^I$)
  iff $I \not\models p(\boldt)$.
\end{proof}

The first item of the Proposition~\ref{prop:ht.facts} means that, when a theory is standard and has no intensional functions, our satisfaction relation is equivalent to the standard satisfaction relation in QEL~\cite{peaval08a}.
The third item of the Proposition~\ref{prop:ht.facts} implies that~${\tuple{H,I} \modelsht \sneg\sneg p(\boldt)}$ iff~${\tuple{H,I} \modelsht \neg\neg p(\boldt)}$ does not hold even when~$\boldt$ does not contain intensional symbols.
%
%
This behavior is consistent with the way a straightforward generalization of the FLP\nobreakdash-reduct~\cite{fapfle11a} treats double negation~\cite{harlif19a}.
%


\subsection{Proof of Section Correspondence with Clingo and DLV}

\subsection{Grounding}

\begin{lemma}\label{lem:grounding.cl}
    An interpretation $I$ satisfies a sentence~$F$
    over~$\sigma^I$ iff $I$
    satisfies $\g IF$.
\end{lemma}

\begin{proof}
By induction on the size of~$F$.
\\[5pt]
\emph{Case 1}: $F$ is an atomic sentence that contains intensional symbols. Then, $\g IF = F$ and the result is trivial.
\\[5pt]
\emph{Case 2}: $F$ is an atomic sentence that does not contain intensional symbols. Then, $\g IF = \top$ if $I \models F$ and $\g IF = \bot$ otherwise. The result follows immediately.
\\[5pt]
\emph{Case 3}: $F$ is of the form~$\sneg G$. Then, $\g IF = \sneg \g IG $ and the result follows by induction hypothesis.
\\[5pt]
\emph{Case 4}: $F$ is $\forall X G(X)$ with $X$ a variable of sort~$s$.
Then, $\g IF = \{ \g{I}{G(d^*)} \mid d^* \in |I|^s \}^\wedge$ and
\\\phantom{iff}~${\tuple{H,I} \modelsht F}$
\\iff~${\tuple{H,I} \modelsht G(d^*)}$ for each~$d\in|I|^s$ 
\\iff~${\tuple{H,I} \modelsht \g{I}{G(d^*)}}$ for each~$d\in|I|^s$ (induction)
\\iff~$\tuple{H,I} \modelsht \g IF$.
\\[5pt]
The case where~$F$ is $\exists X G(X)$ is analogous to Case~2. The remaining cases where~$F$ is~$G_1 \wedge G_2$, $G_1 \vee G_2$ or $F_1 \to F_2$ follow immediately by induction.
\end{proof}

\begin{proof}[Proof of Proposition~\ref{lem:grounding.ht}]
    By induction on~$F$ similar to Lemma~\ref{lem:grounding.cl}.
    \\[5pt]
    \emph{Case 1}: $F$ is an implication of the form~${G_1 \to G_2}$.
    Then, $\g IF$ is the implication~${\gr I{G_1} \to \g I{G_2}}$ and
    \\\phantom{iff}~${\tuple{H,I} \modelsht F}$
    \\iff~$I \models F$ and either~$\tuple{H,I} \not\modelsht G_1$ or~$\tuple{H,I} \modelsht G_2$
    \\iff~$I \models \g IF$ (Lemma~\ref{lem:grounding.cl}) and
    \\\phantom{iff} either~$\tuple{H,I} \not\modelsht \gr I{G_1}$
      or~$\tuple{H,I} \modelsht \g I{G_2}$
    \\iff~$\tuple{H,I} \modelsht \g IF$.
    \\[3pt]
    \emph{Case 2}: $F$ is of the form~$\sneg G$. Then, $\g IF = \sneg \g IG$ and $\tuple{H,I} \models \sneg G$
    \\iff $I \not\models G$ and~$H \not\models G$
    \\iff $I \not\models \g IG$ and~$H \not\models \g IG$ \hfill (Lemma~\ref{lem:grounding.cl})
    \\iff $\tuple{H,I} \models \sneg \g IG$
    \\[3pt]
    The other cases follow by induction as in Lemma~\ref{lem:grounding.cl}.
\end{proof}
\subsection{FT-reduct}

\begin{proof}[Proof of Proposition~\ref{prop:infinitary.interpretations}]
    We proceed by induction on the rank $r$ of $F$.
    For a formula $F$ of rank $r+1$, assume that, for all formulas $G$ of lesser rank than $F$ occurring in $F$, $\tuple{H,I} \modelsht G$ iff $\mathcal{H} \models G^{\mathcal{I}}$.
    \\[5pt]
    \emph{Base Case:} $r=0$, $F$ is a ground atomic formula.
        Then,
        ${\tuple{H,I} \modelsht F}$
        \\iff~${H \models F}$ and~${I\models F}$
        \\iff ${H\models F}$ and~${F^{\mathcal{I}} = F}$
        \\iff ${\mathcal{H} \models F^{\mathcal{I}}}$
    \\[5pt]
    \emph{Induction Step:}
    \\
    \emph{Case 1}: Formula $F$ of rank $r+1$ has form $\HF^\land$.
    Then,
        ${\tuple{H,I} \modelsht F}$
        \\iff ${\tuple{H,I} \modelsht G}$ for every formula $G$ in~$\HF$ \hfill(by definition)
        \\iff $\mathcal{H} \models G^{\mathcal{I}}$ for every formula $G$  in~$\HF$ \hfill(by induction)
        \\iff ${\mathcal{H} \models \{G^{\mathcal{I}} |\ G\in\HF\}^\land}$
        \\iff ${\mathcal{H} \models F^{\mathcal{I}}}$
    \\[3pt]
    \emph{Case 2}: Formula $F$ of rank $r+1$ has form $\HF^\lor$.
    Then,
        ${\tuple{H,I} \modelsht F}$
        \\iff ${\tuple{H,I} \modelsht G}$ for some formula $G$ in~$\HF$ \hfill(by definition)
        \\iff $\mathcal{H} \models G^{\mathcal{I}}$ for this certain formula $G$  in~$\HF$ \hfill(by induction)
        \\iff ${\mathcal{H} \models \{G^{\mathcal{I}} |\ G\in\HF\}^\lor}$
        \\iff ${\mathcal{H} \models F^{\mathcal{I}}}$
    \\[3pt]
    \emph{Case 3}: Formula $F$ of rank $r+1$ has form $G_1 \rar G_2$.
    Then,
        ${\tuple{H,I} \modelsht F}$
        \\iff ${I \models G_1\to G_2}$ and $\langle H, I\rangle \not\modelsht G_1$ or $\langle H, I\rangle \modelsht G_2$ \hfill(by definition)
        \\iff ${\tuple{I,I} \modelsht G_1\to G_2}$ and
        \\\phantom{iff} $\langle H, I\rangle \not\modelsht G_1$ or $\langle H, I\rangle \modelsht G_2$ \hfill(by definition)
        \\iff ${\mathcal{I} \models G_1^{\mathcal{I}}\to G_2^{\mathcal{I}}}$ and $\mathcal{H} \not\models G_1^{\mathcal{I}}$ or $\mathcal{H} \models G_2^{\mathcal{I}}$ \hfill(by induction)
        \\iff ${\mathcal{I} \models G_1\to G_2}$ and $\mathcal{H} \not\models G_1^{\mathcal{I}}$ or $\mathcal{H} \models G_2^{\mathcal{I}}$
        \\iff ${\mathcal{I} \models G_1\to G_2}$ and ${\mathcal{H} \models G_1^{\mathcal{I}} \to G_2^{\mathcal{I}}}$
        \\iff $F^{\mathcal{I}} = G_1^{\mathcal{I}} \to G_2^{\mathcal{I}}$
        and ${\mathcal{H} \models G_1^{\mathcal{I}} \to G_2^{\mathcal{I}}}$
        \\iff ${\mathcal{H} \models F^{\mathcal{I}}}$.
\end{proof}

\subsubsection{The $\tau$ translation.}

The $\tau$ translation transforms a logic program into an infinitary propositional formula~\cite{gehakalisc15a}.
For any ground atom~$A$, it is defined as follows:
\begin{itemize}
    \item $\tau(A) \text{ is } A$,
    \item $\tau(\text{not } A) \text{ is } \neg A$, and
    \item $\tau(\text{not not } A) \text{ is } \neg\neg A$.
\end{itemize}
 For a comparison symbol~$\prec$ and ground terms~$t_1$ and~$t_2$, it is defined as follows:
\begin{itemize}
    \item $\tau(t_1 \prec t_2) \text{ is } \top$ if the relation~$\prec$ holds between~$t_1$ and~$t_2$ and $\bot$ otherwise.
\end{itemize}
For an aggregate element~$E$ of the form of~\eqref{eq:agel:2} with~$\boldY$ the list of local variables occurring on it, $\Psi_{E}$ denotes the set of tuples~$\boldy$ of ground program terms of the same length as~$\boldY$.
We say that a subset~$\Delta$ of~$\Psi_{E}$ \emph{justifies} aggregate atom~$\mathtt{op}\{ E \} \prec u$ if the relation~$\prec$ holds between~$\hat{\mathtt{op}}[\Delta]$ and~$u$ where~$[\Delta] = \{ \boldt^\boldY_\boldy \mid \boldy \in \Delta \}$ and~$\boldt$ is the list of terms~$t_1, \dots, t_k$ in~$E$.
For an aggregate atom~$A$ of the form of~$\mathtt{op}\{ E \} \prec u$ with global vriables~$\boldX$, $\tau(A)$ it is defined as the infinitary formula
\begin{gather}
    \bigwedge_{\Delta \in \chi}
    \left(
        \bigwedge_{\boldy \in \Delta}
        \boldl^{\boldX\boldY}_{\boldx\boldy}
        \to
        \bigvee_{\boldy \in \Psi_{E^\boldX_\boldx} \setminus \Delta}
        \boldl^{\boldX\boldY}_{\boldx\boldy}
    \right)
    \label{eq:1:lem:abstrac.gringo.correspondence.ht}
\end{gather}
where~$\chi$ is the set of subsets~$\Delta$ of~$\Psi_{E}$ that do not justify aggregate atom~$\mathtt{op}\{ E^\boldX_\boldx \} \prec u$, and~$\boldl$ is the list~$l_1,\dotsc,l_m$ of literals in~$E$.
We omit the parentheses and write~$\tau F$ instead of~$\tau(F)$ when clear.
For a rule~$R$ of the form of~\eqref{rule:4} with global variables~$\boldZ$, $\tau R$ is the infinitary conjunction of all formulas of the form
\begin{align}
    \tau (B_1)^\boldZ_\boldz \wedge \dotsc \wedge \tau (B_n)^\boldZ_\boldz \to \tau \mathit{Head}^\boldZ_\boldz 
\end{align}
with~$\boldz$ being a list of ground program terms of the same length as~$\boldZ$.
For a program~$\Pi$, $\tau\Pi = \{ \tau R \mid R \in \Pi \}^\wedge$.

\subsubsection{Correspondence with \clingo.}

For any rule~$R$ without aggregates, it is not difficult to see that $\tau R = \g I {\taug R}$ for any standard interpretation~$I$.
For rules with aggregates $\tau R$ and~$\g I {\taug R}$ only differ in the translation of aggregates.
The following two results show the relation between~$\tau A$ and~$\g I {\taug A}$ for any aggregate atom~$A$.

\begin{lemma}\label{lem:abstrac.gringo.correspondence.cl}
    Let~$I$ be an \agginterp{}, $\mathtt{op}$ be an operation name and~$x \in \{\mathit{cli}, \mathit{dlv} \}$.
    Then, $I$ satisfies $\mathtt{op}(\sets^x_{|E/\boldX|}(\boldx)) \prec u$ iff~$I$ satisfies~\eqref{eq:1:lem:abstrac.gringo.correspondence.ht}.
\end{lemma}

\begin{proof}
    Let~$\boldY$ be the list of variables occurring in~$E$ that do not occur in~$\boldX$.
    Let~$\Delta_I = \{ \, \boldy \in \Psi_{E} \mid I \models \boldl^{\boldX\boldY}_{\boldx\boldy} \, \}$ and~$F_I$ be the formula
    \begin{small}
        \begin{gather*}
            \bigwedge_{\boldy \in \Delta_I}
            \boldl^{\boldX\boldY}_{\boldx\boldy}
            \to
            \bigvee_{\boldy \in \Psi_E \setminus \Delta_I}
            \boldl^{\boldX\boldY}_{\boldx\boldy}
        \end{gather*}
    \end{small}
    Then, $I \not\models F_I$ and
    $$\sets_{|E/\boldX|}(\boldx)^I  = \{ \boldt^{\boldX\boldY}_{\boldx\boldy} \mid \boldy \in \Delta_I \} = [\Delta_I].$$
    Consequently, we have
    \begin{align*}
        I \models \eqref{eq:1:lem:abstrac.gringo.correspondence.ht} &\text{ iff } F_I \text{ is not a conjunctive term of } \eqref{eq:1:lem:abstrac.gringo.correspondence.ht}
        \\
        &\text{ iff } \Delta_I \text{ justifies } \mathtt{op}(E^\boldX_\boldx) \prec u
        \\
        &\text{ iff } \hat{\mathtt{op}}[\Delta_I] \prec u
        \\
        &\text{ iff } \hat{\mathtt{op}}(\sets_{|E/\boldX|}(\boldx)^I) \prec u
        \\
        &\text{ iff } I \models \mathtt{op}(\sets_{|E/\boldX|}(\boldx)) \prec u
        \qedhere
    \end{align*}
    \let\qed\relax
\end{proof}

\begin{lemma}\label{lem:abstrac.gringo.correspondence.ht.agg}
    Let~$\mathtt{op}$ be an operation name.
    Then, an \agghtinterp~$\tuple{H,I}$ satisfies ${\mathtt{op}(\setsg_{|E/\boldX|}(\boldx)) \prec u}$ iff~$\tuple{H,I}$ satisfies~\eqref{eq:1:lem:abstrac.gringo.correspondence.ht}.
\end{lemma}

\begin{proof}
    Let us denote~$\mathtt{op}(\setsg_{|E/\boldX|}(\boldx)) \prec u$ as~$A$ in the following.
    \\[10pt]
    \emph{Case 1}: ${I \not\models A}$.
    By Lemma~\ref{lem:abstrac.gringo.correspondence.cl}, it follows that~${I \not\models  \eqref{eq:1:lem:abstrac.gringo.correspondence.ht}}$.
    Therefore, ${\tuple{H,I} \not\modelsht  A}$ and~${\tuple{H,I} \not\modelsht  \eqref{eq:1:lem:abstrac.gringo.correspondence.ht}}$.
    \\[10pt]
    \emph{Case 2}: ${I \models A}$.
    By Lemma~\ref{lem:abstrac.gringo.correspondence.cl}, it follows that~${I \models  \eqref{eq:1:lem:abstrac.gringo.correspondence.ht}}$.
    %
    %
    Let
    $$\Delta_{\tuple{H,I}} = \{ \, \boldy \in \Psi_{E} \mid \tuple{H,I} \modelsht \boldl^{\boldX\boldY}_{\boldx\boldy} \, \}$$
     and~$F_{\tuple{H,I}}$ be the formula
    \begin{small}
        \begin{gather*}
            \bigwedge_{\boldy \in \Delta_{\tuple{H,I}}}
            \boldl^{\boldX\boldY}_{\boldx\boldy}
            \to
            \bigvee_{\boldy \in \Psi_E \setminus \Delta_{\tuple{H,I}}}
            \boldl^{\boldX\boldY}_{\boldx\boldy}
        \end{gather*}
    \end{small}
    Then, $\tuple{H,I} \not\models F_{\tuple{H,I}}$ and
    $$\sets_{|E/\boldX|}(\boldx)^H  = \{ \boldt^{\boldX\boldY}_{\boldx\boldy} \mid \boldy \in \Delta_{\tuple{H,I}} \} = [\Delta_{\tuple{H,I}}].$$
    Consequently, we have
    \begin{align*}
        \tuple{H,I} \models \eqref{eq:1:lem:abstrac.gringo.correspondence.ht} &\text{ iff }I \models \eqref{eq:1:lem:abstrac.gringo.correspondence.ht} \text{ and}
        \\
        &\text{ \phantom{iff} }  F_{\tuple{H,I}} \text{ is not a conjunctive term of } \eqref{eq:1:lem:abstrac.gringo.correspondence.ht}
        \\
        &\text{ iff } \Delta_{\tuple{H,I}} \text{ justifies }  \mathtt{op}(\sets_{|E/\boldX|}(\boldx)) \prec u
        \\
        &\text{ iff } \hat{\mathtt{op}}([\Delta_{\tuple{H,I}}]) \prec u
        \\
        &\text{ iff } \hat{\mathtt{op}}(\setsg_{|E/\boldX|}(\boldx)^H) \prec u
        \\
        &\text{ iff } H \models \mathtt{op}(\setsg_{|E/\boldX|}(\boldx)) \prec u
        \\
        &\text{ iff } H \models A
        \\
        &\text{ iff } \tuple{H,I} \models A \qedhere
    \end{align*}
    \let\qed\relax
\end{proof}

\begin{lemma}\label{lem:abstrac.gringo.correspondence.ht}
    Let~$\Pi$ be a program, $\P$ be the set of all predicate symbols in~$\sigma$ other than comparisons, $\F$ be the set of all function symbols corresponding set symbols.
    Then,
    $${\tuple{H,I} \modelsht\tau\Pi} \quad\text{ iff }\quad {\tuple{H,I} \modelsht\gr{I}{\taug\Pi}}$$
    for every~\agghtinterp~$\tuple{H,I}$.
\end{lemma}

\begin{proof}
    Recall that comparisons are not intensional in the definition of the stable models of a program, that is, they do not belong to~$\P$.
    Then, it is easy to see that~$\tau\Pi$ can be obtained from~$\gr{\Int}{\taug\Pi}$
    by replacing each occurrence of~${\mathit{op}(\setsg_{|E/\boldX|}(\boldx)) \prec u}$, where~$\mathit{op}$ is an operation name,
    by its corresponding formula of the form of~\eqref{eq:1:lem:abstrac.gringo.correspondence.ht}:
    %
    %
    Hence, it is enough to show that
    \begin{small}
    \begin{gather*}
        \tuple{H,I} \modelsht \mathit{op}(\setsg_{|E/\boldX|}(\boldx) \prec u)
        \quad\text{iff}\quad
        \tuple{H,I} \modelsht\eqref{eq:1:lem:abstrac.gringo.correspondence.ht}
    \end{gather*}
    \end{small}
    This follows from Lemma~\ref{lem:abstrac.gringo.correspondence.ht.agg}.
\end{proof}

\begin{lemma}\label{lem:abstrac.gringo.correspondence.ht2b}
    Let~$\Pi$ be a program, $\P$ be the set of all predicate symbols in~$\sigma$ other than comparisons, $\F$ be the set of all function symbols corresponding set symbols.
    Then,
    $${\tuple{H,I} \modelsht \tau\Pi} \quad\text{ iff }\quad {\tuple{H,I} \modelsht \taug\Pi}$$
    for every~\agghtinterp~$\tuple{H,I}$.
\end{lemma}

\begin{proof}
    Directly by Proposition~\ref{lem:grounding.ht} and Lemma~\ref{lem:abstrac.gringo.correspondence.ht}.
\end{proof}

\begin{lemma}\label{lem:abstrac.gringo.correspondence.ht2}
    Let~$\Pi$ be a program, $\P$ be the set of all predicate symbols in~$\sigma$ other than comparisons, $\F$ be the set of all function symbols corresponding set symbols.
    Then,
    $${\Ans{H} \models(\tau\Pi)^{\Ans{I}}} \quad\text{ iff }\quad {\tuple{H,I} \modelsht \taug\Pi}$$
    for every~\agghtinterp~$\tuple{H,I}$.
\end{lemma}

\begin{proof}
    Since~$\tau\Pi$ is an infinitary formula of~$\sigma^p$, the result follows by Lemma~\ref{lem:abstrac.gringo.correspondence.ht2b} and Proposition~\ref{prop:infinitary.interpretations}.
\end{proof}


\begin{lemma}\label{lem:agginterp.less}
    Let~$\tuple{H,I}$ be an \agghtinterp.
    Then,~$H \less I$ iff $H \lessP{\P\emptyset} I$.
\end{lemma}

\begin{proof}
    \emph{Right-to-left}.
    $H \lessP{\P\emptyset} I$ means that there is a predicate symbol~$p$ such that~$p^H \subset p^I$ and, thus,~$H \less I$ also holds.
    \emph{Lert-to-right}.
    $H \less I$ means that one of the following holds
    \begin{itemize}
        \item $p^H \subset p^I$ for some intensional predicate symbol~$p$; or
        \item $f^H \neq f^I$ for some intensional function symbol~$f$.
    \end{itemize}
    The first immediately implies that~${H \lessP{\P\emptyset} I}$ also holds.
    For the latter, $f$ must be of the form~$\sets^x_{|E/\boldX|}$ for some aggregate element~$E$.
    Therefore, the set of the set of all tuples of form~$\tuple{(t_1)^{\boldX\boldY}_{\boldx\boldy},\dots,(t_k)^{\boldX\boldY}_{\boldx\boldy}}$ such that~$I$ satisfies~$(l_1)^{\boldX\boldY}_{\boldx\boldy} \wedge \dots \wedge (l_m)^{\boldX\boldY}_{\boldx\boldy}$ and the set of all tuples of form~$\tuple{(t_1)^{\boldX\boldY}_{\boldx\boldy},\dots,(t_k)^{\boldX\boldY}_{\boldx\boldy}}$ such that~$\tuple{H,I}$ or~$H$ satisfies~$(l_1)^{\boldX\boldY}_{\boldx\boldy} \wedge \dots \wedge (l_m)^{\boldX\boldY}_{\boldx\boldy}$ must be different.
    This means that~${p^H \neq p^I}$ for some predicate symbols~$p$ and, thus, ${p^H \subset p^I}$ and~${H \lessP{\P\emptyset} I}$ follow.
\end{proof}

\begin{proof}[Proposition~\ref{prop:less.ht}]
    By Lemma~\ref{lem:agginterp.less}, it follows that~${H \less I}$ iff~${H \lessP{\P\emptyset} I}$.
    Then, the result follows because the latter holds iff~$\mathcal{H} \subset \mathcal{I}$.
\end{proof}

\begin{proof}[Proof of Theorem~\ref{thm:abstrac.gringo.correspondence}]
Assume that~$\Ans{I}$ is a fo\nobreakdash-clingo answer set of~$\Pi$.
By definition, there is a $I$ is an \aggstable model of~$\taug \Pi$.
In its turn, this implies that~$I$ is an \aggmodel of~$\taug \Pi$
and there is no \agghtmodel~$\tuple{H,I}$ of~$\taug \Pi$ with~$H \less I$.
By Lemma~\ref{lem:abstrac.gringo.correspondence.ht2},
it follows that~$\Ans{I}$ is a model of~$(\tau\Pi)^{\Ans{I}}$.
Suppose, for the sake of contradiction, that there is~${\Ans{H} \subset \Ans{I}}$ such that~$\Ans{H} \models (\tau\Pi)^{\Ans{I}}$.
Let~$\tuple{H,I}$ be the \agginterp with~$\Ans{H}$ and~$\Ans{I}$ the set of ground atoms of~$\sigma^p$ satisfied by~$H$ and~$I$, respectively.
%
%
Then, by Proposition~\ref{prop:less.ht} and Lemma~\ref{lem:abstrac.gringo.correspondence.ht2},
it follows~${H \less I}$ and~${\tuple{H,I} \modelsht \taug \Pi}$.
This is a contradiction because there is no \agghtmodel~$\tuple{H,I}$ of~$\taug \Pi$ with~$H \less I$.
\\[5pt]
Conversely, assume that~$\Ans{I}$ is a clingo answer set of~$\Pi$.
By definition, $\Ans{I}$ is a model of~$\tau\Pi$ and there is no model~$\Ans{H}$ of~$(\tau\Pi)^{\Ans{I}}$ with~$\Ans{H} \subset \Ans{I}$.
By Lemma~\ref{lem:abstrac.gringo.correspondence.ht2}, the former implies that there is an \aggmodel~$I$ of~$(\taug\Pi)^{\Ans{I}}$.
Suppose, for the sake of contradiction, that there is some \agghtmodel~$\tuple{H,I}$ of~$\taug\Pi$ with~$H \less I$.
By Lemma~\ref{lem:abstrac.gringo.correspondence.ht2} and Proposition~\ref{prop:less.ht}, this implies that~$\Ans{H}$ satisfies~$(\taug\Pi)^{\Ans{I}}$ with~~$\Ans{H} \subset \Ans{I}$, which is a contradiction.
\end{proof}
\subsubsection{Correspondence with \dlv.}

A \emph{dlv\nobreakdash-literal} is either an atomic formula, a truth constant ($\top$ or $\bot$) or an expression of the forms~$\sneg A$, $\sneg\sneg A$ with~$A$ an atomic formula.
A \emph{dlv\nobreakdash-implication} is an implication of the form~${F_1 \to F_2}$ where~$F_1$ is a conjunction of dlv\nobreakdash-literals and~$F_2$ is either an atomic formula or the truth constant~$\bot$.

For any formula~$F$, by~$\NN{F}$ we denote the result of replacing all occurrences of~$\sneg$ by~$\neg$ in~$F$.
If~$F$ is an infinitary propositional formula, then~$\NN{F}$ is standard.

\begin{lemma}
    \label{lem:n.satisfaction}
    Let~$F$ be an infinitary propositional formula.
    Then, $I \models F$ iff~$\Ans{I} \models \NN{F}$.
\end{lemma}

\begin{proof}
    By definition, ${\Ans{I} \models \neg G}$ iff~${\Ans{I} \models \sneg G}$.
    Then, by induction, it follows that~${\Ans{I} \models F}$ iff~${\Ans{I} \models \NN{F}}$.
    Finally, since~$\NN{F}$ is standard, by Proposition~\ref{prop:infinitary.interpretations}, we get that the latter holds iff~${I \models F}$.
\end{proof}

Given an implication~$F$ of the form~${F_1 \to F_2}$,
by~${\PPP{F}}$ we denote the implication~${\sneg\sneg F_1 \to F_2}$ and by~${\PNN{F}}$ we denote the implication~${\sneg\sneg\NN{F_1} \to F_2}$.
Note that~${I \models F}$ iff~${I \models \PPP{F}}$ iff~${I \models \PNN{F}}$.
For an \htinterp these equivalences do not hold, but we have the following interesting relationship with the FLP\nobreakdash-reduct.

\begin{lemma}
    \label{lem:flp.reduct.aux}
    Let~$\tuple{H,I}$ be an \htinterp and~$F$ be an infinitary propositional formula of the form~${F_1 \to F_2}$
    with~$F_2$ an atomic formula or a truth constant.
    Then, ${\tuple{H,I} \modelsht \PPP{F}}$ iff~${\Ans{I} \models F}$ and~${\Ans{H} \models \FLP(F,\Ans{I})}$.
\end{lemma}

\begin{proof}
    %
    \emph{Left-to-right.}
    Assume~${\tuple{H,I} \modelsht \PPP{F}}$.
    By Proposition~\ref{prop:persistence}, we get~${I \models {F}}$ and, thus~${I \models \PPP{F}}$.
    By Proposition~\ref{prop:infinitary.interpretations}, this implies~${\Ans{I} \models F}$.
    \emph{Case~1.} ${\Ans{I} \not\models F_1}$.
    Then, ${\Ans{I} \models F}$ and ${\FLP(F,\Ans{I}) = \top}$, and the results immediately holds.
    \emph{Case~2.} ${\Ans{I} \models F_1}$.
    Then, ${\FLP(F,\Ans{I}) = F}$.
    If~${\Ans{H} \not\models F_1}$, then the result follows immediately.
    Otherwise, ${\Ans{H} \models F_1}$, and this implies~${\tuple{H,I} \modelsht \sneg\sneg F_1}$ (by Propositions~\ref{prop:infinitary.interpretations} and~\ref{prop:ht:facts:abbr:neg}, and fact~${\Ans{I} \models F_1}$).
    Since~${\tuple{H,I} \modelsht \PPP{F}}$, this implies that~${\tuple{H,I} \modelsht F_2}$.
    Hence, $F_2$ is not~$\bot$ and~${\Ans{H} \models F_2}$.
    This means that~${\Ans{H} \models \FLP(F,\Ans{I})}$.
    \\[5pt]
    \emph{Right-to-left.}
    Assume~${\Ans{I} \models F}$ and~${\Ans{H} \models \FLP(F,\Ans{I})}$.
    By Proposition~\ref{prop:infinitary.interpretations}, we get~${I \models {F}}$ and, thus,~${I \models \PPP{F}}$.
    We proceed by cases.
    \emph{Case~1.}
    ${\Ans{I} \not\models F_1}$.
    By Proposition~\ref{prop:infinitary.interpretations}, this implies~${I \not\models \sneg\sneg F_1}$ and, by Proposition~\ref{prop:persistence}, it follows~${\tuple{H,I} \not\modelsht \sneg\sneg{F_1}}$.
    Hence, ${\tuple{H,I} \modelsht \PPP{F}}$.
    \emph{Case~2.}
    ${\Ans{I} \models F_1}$.
    Then, ${I \models \sneg\sneg{F_1}}$ (by Proposition~\ref{prop:infinitary.interpretations}) and~${\FLP(F,\Ans{I}) = F}$ (by definition).
    %
    %
    The latter implies~${\Ans{H} \models F}$.
    If~${\tuple{H,I} \not\models \sneg\sneg{F_1}}$, then the result follows because~${I \models \PPP{F}}$.
    Otherwise, ${\tuple{H,I} \modelsht \sneg\sneg{F_1}}$ and this implies~${\Ans{H} \models F_1}$ (by Proposition~\ref{prop:ht:facts:abbr:neg}).
    Since~${\Ans{H} \models F}$, this implies that~${\Ans{H} \models F_2}$.
    Then~$F_2$ is not~$\bot$ and~${\tuple{H,I} \modelsht F_2}$.
    Since~${I \models F}$, this implies~${\tuple{H,I} \modelsht F}$.
\end{proof}

\begin{lemma}
    \label{lem:dlv.literal.satisfaction}
    Let~$\tuple{H,I}$ be an \htinterp and~$L$ be a dlv\nobreakdash-literal.
    Then, ${\tuple{H,I} \modelsht L}$ iff~${I \models L}$ and~${H \models L}$.
\end{lemma}

\begin{proof}
    If~$L$ is an atomic formula or a truth constant, the result holds by definition.
    Otherwise, it follows by Proposition~\ref{prop:ht:facts:abbr:neg}.
\end{proof}

\begin{lemma}
    \label{lem:dlv.body.satisfaction}
    Let~$\tuple{H,I}$ be an \htinterp and~$F$ be a conjunction of dlv\nobreakdash-literals.
    Then, ${\tuple{H,I} \modelsht F}$ iff~${I \models F}$ and~${H \models F}$.
\end{lemma}

\begin{proof}
    Let $F = L_1 \land \ldots \land L_n$.
    Then, ${\tuple{H,I} \modelsht F}$ iff (by definition) ${\tuple{H,I} \modelsht L_i}$ for all~$1 \leq i \leq n$ iff (Lemma~\ref{lem:dlv.literal.satisfaction}) ${H \models L_i}$ and~${I \models L_i}$ for all~$1 \leq i \leq n$
    iff ${H \models L_i}$ for all~$1 \leq i \leq n$ and~${I \models L_i}$ for all~$1 \leq i \leq n$
    iff (by definition) ${H \models F}$ and~${I \models F}$.
\end{proof}

\begin{lemma}
    \label{lem:dlv.implication.rewriting}
    Let~$\tuple{H,I}$ be an \htinterp and~${F}$ be a dlv\nobreakdash-implication.
    Then, ${\tuple{H,I} \modelsht F}$ iff~${\tuple{H,I} \modelsht \PNN{F}}$.
\end{lemma}

\begin{proof}
    Let~$F$ be of the form~${F_1 \to F_2}$.
    Note that ${I \models F_1}$ iff~${I \models \sneg\sneg \NN{F_1}}$.
    Hence, ${I \models F}$ iff~${I \models \PNN{F}}$.
    Therefore,
    ${\tuple{H,I} \modelsht F_1 \to F_2}$
    iff~${\tuple{H,I} \not\modelsht F_1}$ or~${\tuple{H,I} \models F_2}$
    iff (Lemma~\ref{lem:dlv.body.satisfaction}) ${I\not\models F_1}$ or~${H \not\models F_1}$ or~${\tuple{H,I} \models F_2}$
    iff (Lemma~\ref{lem:n.satisfaction}) ${I\not\models \NN{F_1}}$ or~${H \not\models \NN{F_1}}$ or~${\tuple{H,I} \models F_2}$
    iff (Proposition~\ref{prop:ht:facts:abbr:neg}) ${\tuple{H,I} \not\modelsht \sneg\sneg \NN{F_1}}$ or~${\tuple{H,I} \models F_2}$
    iff~${\tuple{H,I} \modelsht \sneg\sneg \NN{F_1} \to F_2}$.
\end{proof}

\begin{lemma}
    \label{lem:dlv.translations}
    Let~$R$ be a rule and~$\tuple{H,I}$ be an \agghtinterp. Then,
    $${\tuple{H,I} \modelsht \PPP{\tau R} \quad\text{ iff }\quad\tuple{H,I} \modelsht \PNN{R'}}$$
    where~$R'$ is~$\gr{I}{\taud R}$.
\end{lemma}

\begin{proof}
    We can see that~$\PPP{\tau R}$ and~$\PNN{R'}$ are of the form of~${\sneg\sneg F_1 \to F_2}$ and~${\sneg\sneg F_1' \to F_2}$, respectively, with~$F_1$ and~$F_1'$ differing only in the translation of aggregates, with the former containing formula~\eqref{eq:1:lem:abstrac.gringo.correspondence.ht} where the latter contains an atom of the form~$\mathtt{op}(\setsd_{|E/\boldX|}(\boldx)) \prec u$.
    By Proposition~\ref{prop:ht:facts:abbr:neg},
    it follows that~${\tuple{H,I} \modelsht \sneg\sneg F_1}$ iff~$I \models F_1$ and~$H \models F_1$, and~${\tuple{H,I} \modelsht \sneg\sneg F_1'}$ iff~$I \models F_1'$ and~$H \models F_1'$.
    Finally, by Lemma~\ref{lem:abstrac.gringo.correspondence.cl}, we get~$I \models F_1$ iff~$I \models F_1'$ and~$H \models F_1$ iff~$H \models F_1'$.
\end{proof}

\begin{lemma}\label{lem:dlv.translations.cl}
    Let~$R$ be a rule and~$I$ be an \agginterp. Then,
    $${I \models \tau R \quad\text{ iff }\quad I \models \taud R}$$
\end{lemma}

\begin{proof}
    Let~$R'$ be the result of~$\gr{I}{\taud R}$.
    Then, it follows that
    ${I \models \tau R}$
    iff~${I \modelsht \PPP{\tau R}}$
    iff~${\tuple{I,I} \modelsht \PPP{\tau R}}$
    iff (Lemma~\ref{lem:dlv.translations})~${\tuple{I,I} \modelsht \PNN{R'}}$
    iff~${I \modelsht \PNN{R'}}$
    iff~${I \modelsht R'}$
    iff (Proposition~\ref{lem:grounding.cl})~${I \models \taud R}$.%
\end{proof}

\begin{lemma}\label{lem:dlv.reduct.ht}
    Let~$R$ be a rule and~$\tuple{H,I}$ be an \agghtinterp.
    Then, the following two conditions are equivalent
    \begin{itemize}
        \item ${\Ans{I}\models \tau R}$ and ${\Ans{H} \models \FLP(\tau R,\Ans{I})}$, and
        \item ${\tuple{H,I} \modelsht \taud R}$.
    \end{itemize}
\end{lemma}

\begin{proof}
    Let~$R'$ be~$\gr{I}{\taud R}$.
    By Proposition~\ref{lem:grounding.ht},
    we get~${\tuple{H,I} \modelsht \taud R}$ iff~${\tuple{H,I} \models R'}$.
    Furthermore,
    $R'$ is a dlv\nobreakdash-implication and,
    by Lemmas~\ref{lem:flp.reduct.aux} and~\ref{lem:dlv.implication.rewriting} we respectively get:
    \begin{itemize}
        \item ${\tuple{H,I} \modelsht \PPP{\tau R}}$ iff~${\Ans{I} \models F}$ and~${\Ans{H} \models \FLP(\tau R,\Ans{I})}$,
        \item ${\tuple{H,I} \modelsht R'}$ iff~${\tuple{H,I} \modelsht \PNN{R'}}$.
    \end{itemize}
    Hence, it remains to be shown
    $${\tuple{H,I} \modelsht \PPP{\tau R} \quad\text{ iff }\quad\tuple{H,I} \modelsht \PNN{R'}}$$
    which follows by Lemma~\ref{lem:dlv.translations}.
\end{proof}

\begin{proof}[Proof of Theorem~\ref{thm:dlv.correspondence}]
    Assume that~$\Ans{I}$ is a fo\nobreakdash-dlv answer set of~$\Pi$.
    By definition, there is a $I$ is an \aggstable model of~$\taud \Pi$.
    In its turn, this implies that~$I$ is an \aggmodel of~$\taud \Pi$
    and there is no \agghtmodel~$\tuple{H,I}$ of~$\taud \Pi$ with~${H \less I}$.
    By Lemma~\ref{lem:dlv.translations.cl},
    it follows that~$\Ans{I}$ is a model of~$\tau\Pi$.
    Suppose, for the sake of contradiction, that there is~${\Ans{H} \subset \Ans{I}}$ such that~${\Ans{H} \models \FLP(\tau\Pi,\Ans{I})}$.
    Let~$\tuple{H,I}$ be the \agghtinterp with~$\Ans{H}$ and~$\Ans{I}$ the set of ground atoms of~$\sigma^p$ satisfied by~$H$ and~$I$, respectively.
    %
    %
    Then, by Proposition~\ref{prop:less.ht} and Lemma~\ref{lem:dlv.reduct.ht},
    it follows~${H \less I}$ and~${\tuple{H,I} \modelsht \taud \Pi}$.
    This is a contradiction because there is no \agghtmodel~$\tuple{H,I}$ of~$\taud \Pi$ with~${H \less I}$.
    \\[5pt]
    Conversely, assume that~$\Ans{I}$ is a dlv answer set of~$\Pi$.
    By definition, $\Ans{I}$ is a model of~$\tau\Pi$ and there is no model~$\Ans{H}$ of~$\FLP(\tau\Pi,\Ans{I})$ with~${\Ans{H} \subset \Ans{I}}$.
    By Lemma~\ref{lem:dlv.translations.cl}, the former implies that there is an \aggmodel~$I$ of~$\taud\Pi$.
    Suppose, for the sake of contradiction, that there is some \agghtmodel~$\tuple{H,I}$ of~$\taud\Pi$ with~${H \less I}$.
    By Lemma~\ref{lem:dlv.reduct.ht} and Proposition~\ref{prop:less.ht}, this implies that~$\Ans{H}$ satisfies~$\FLP(\taud\Pi,\Ans{I})$ with~~$\Ans{H} \subset \Ans{I}$, which is a contradiction.
    \end{proof}

\subsection{Proofs of Section Strong Equivalence}\label{sec:proofs.strong.equivalece}

\begin{lemma}\label{lem:same.htmodels.same.stable.models}
    If~$\Gamma_1$ and~$\Gamma_2$ have the same \agghtmodels, then they have the same \aggstable models.
\end{lemma}

\begin{proof}
    Since~$\Gamma_1$ and~$\Gamma_2$ have the same \agghtmodels, they also have the same \aggmodels.
    Suppose, for the sake of contradiction, that they have different \aggstable models.
    Assume, without loss of generality, that~$I$ is an \aggstable model of~$\Gamma_1$ but not of~$\Gamma_2$.\
    Since~$\Gamma_1$ and~$\Gamma_2$ have the same models and~$I$ is a model of~$\Gamma_1$, it follows that~$I$ is a model of~$\Gamma_2$.
    Since~$I$ is not an \aggstable model of~$\Gamma_2$, there is a \agghtmodel~$\tuple{H,I}$ of~$\Gamma_2$ such that~${H \less I}$.
    Since $\Gamma_1$ and~$\Gamma_2$ have the same \agghtmodels, this implies that $\tuple{H,I}$ is also an \agghtmodel of~$\Gamma_1$, which is a contradiction with the assumption that~$I$ is a \aggstable model of~$\Gamma_1$.
\end{proof}

For any interpretation~$I$, by~$\Delta_I$ we denote the program containing all facts of the form~``$p(\boldt)$'' such that~${I \models \taug(p(\boldt))}$ with~${p \in \P}$.
Similarly, for an \htinterp~$\tuple{H,I}$, by~$\Delta_{\tuple{H,I}}$ we denote the program containing all facts in~$\Delta_I$ plus all rules of the form~``${p(\boldt) \text{\rm\ruleo} q(\boldu)}$''
such that~${I \models \taug (p(\boldt) \wedge q(\boldu))}$ and~${H \not\models \taug (p(\boldt) \vee q(\boldu))}$ with~$p,q \in \P$.

\begin{lemma}\label{lem:strong.equivalence.only.if.cl}
    Take any two sets of sentences, $\Gamma_1$ and $\Gamma_2$ and let~$I$ be a \aggmodel\ of~$\Gamma_1$ that does not satisfy~$\Gamma_2$.
    Then, $I$ is an \aggstable model of~$\Gamma_1 \cup \tau^x\Delta_I$, but not of~$\Gamma_2\cup \tau^x\Delta_I$ with~$x \in \{\mathit{cli},\dlv \}$.
\end{lemma}

\begin{proof}
    By the definition, it follows that~${(A) \in \Pi}$ iff ${I \models \taug A}$ iff ${I \models \taud A}$.
    Note that~${\taug(A) = \taud(A)}$ for all~${A \in \Delta_I}$.
    Thus, $I$ is a model of~$\tau^x\Pi$.
    Furthermore, there is no~$\tuple{H,I}$ with~${H \lessP{\P\emptyset} I}$ satisfies~$\tau^x\Pi$.
    By Lemma~\ref{lem:agginterp.less}, this implies that there is no \agginterp~$H$ with~$H \less I$ such that~$\tuple{H,I}$ satisfies~$\tau^x\Pi$.
    Since~$I$ is also a model of~$\Gamma_1$, it follows that~$I$ is a model of~$\Gamma_1 \cup \tau^x\Pi$ and, thus, it a \aggstable model of~$\Gamma_1 \cup \tau^x\Pi$.
    Since~$I$ does not satisfy~$\Gamma_2$, it follows that~$I$ is not an \aggstable model of~$\Gamma_2 \cup \tau^x\Pi$.
\end{proof}

\begin{lemma}\label{lem:strong.equivalence.only.if.ht}
    Take any two sets of sentences, $\Gamma_1$ and $\Gamma_2$ with the same classical models and let~$\tuple{H,I}$ be an \agghtmodel\ of~$\Gamma_1$ that does not satisfy~$\Gamma_2$.
    Then, $I$ is an \aggstable model of~$\Gamma_3 \cup \tau^x\Delta_{\tuple{H,I}}$, but not of~$\Gamma_1\cup \tau^x\Delta_{\tuple{H,I}}$ with~$x \in \{\mathit{cli},\dlv \}$.
\end{lemma}

\begin{proof}
    First note that~$\taug(\Delta_{\tuple{H,I}}) = \taud(\Delta_{\tuple{H,I}})$.
    Hence, in the following, we do not distinguish between the two.
    Furthermore, 
    $I$ satisfies~$\tau^x(\Delta_{\tuple{H,I}})$ because, by definition, it satisfies the consequent of every rule in~$\Delta_{\tuple{H,I}}$.
    Hence, $I$ is a model of~$\Gamma_1 \cup \tau^x(\Delta_{\tuple{H,I}})$ and~$\Gamma_2 \cup \tau^x(\Delta_{\tuple{H,I}})$.
    To see that~$I$ is an \aggstable model of~$\Gamma_2 \cup \tau^x(\Delta_{\tuple{H,I}})$, 
    suppose for the sake of contradiction that there is an \agginterp~$J$ with~${J \less I}$ such that~$\tuple{J,I}$ satisfies~$\Gamma_2 \cup \tau^x(\Delta_{\tuple{H,I}})$.
    This implies that~$\tuple{J,I}$ satisfies~$\tau^x(\Delta_I)$ and, thus that~${H \lesseqP{\P\emptyset} J}$.
    %
    %
    %
    Furthermore, $H$ must be different from~$J$ because~$\tuple{H,I}$ does not satisfy~$\Gamma_2$ and~$\tuple{J,I}$ does.
    %
    %
    By Lemma~\ref{lem:agginterp.less}, it follows that~ $J \less I$ implies~$J \lessP{\P\emptyset} I $.
    Hence, ${H \lessP{\P\emptyset} J \lessP{\P\emptyset} I}$.
    Let~$p(\boldt)$ be an atom with~${p \in \P}$ such that~${J \models \tau^x(p(\boldt))}$ and~${H \not\models \tau^x(p(\boldt))}$.
    Let~$q(\boldu)$ be an atom with~${p \in \P}$ such that~${I \models \tau^x(q(\boldu))}$ and~${J \not\models \tau^x(q(\boldu))}$.
    Therefore, rule~``$q(\boldu) \text{\rm\ruleo} p(\boldt)$'' belongs to~$\Delta_{\tuple{H,I}}$.
    Let this rule be named~$R$.
    Then,~$\tuple{J,I}$ does not satisfies~$\tau^xR$.
    This implies that~$\tuple{J,I}$ does not satisfy~$\Gamma_2 \cup \tau^x(\Delta_{\tuple{H,I}})$, which is a contradiction with the assumption.
    Hence, $I$ is an \aggstable model of~$\Gamma_2 \cup \tau^x(\Delta_{\tuple{H,I}})$.

    It remains to be shown that~$I$ is not an \aggstable model of~$\Gamma_1 \cup \tau^x(\Delta_{\tuple{H,I}})$.
    We show that~$\tuple{H,I}$ satisfies~$\Gamma_1 \cup \tau^x(\Delta_{\tuple{H,I}})$.
    It is a model of~$\Delta_I$.
    Furthermore, it also satisfies every rule~$R$ in~$\Delta_{\tuple{H,I}}$ of the form ``$q(\boldu) \text{\rm\ruleo} p(\boldt)$'' because~$\tuple{H,I} \not\modelsht \tau^x(p(\boldt))$ and~$I \models \tau^x(q(\boldu))$.
    Hence, $\tuple{H,I}$ satisfies~$\Gamma_1 \cup \tau^x(\Delta_{\tuple{H,I}})$ and, thus, $I$ is not an \aggstable model of~$\Gamma_1 \cup \tau^x(\Delta_{\tuple{H,I}})$.
\end{proof}

\begin{lemma}
    \label{lem:strong.equivalence.only.if}
    If~$\Gamma_1$ and $\Gamma_2$ do not have the same \agghtmodels, then there is some program~$\Delta$ without aggregates nor double negation such that~${\Gamma_1 \cup \tau^x(\Delta)}$ and ${\Gamma_2 \cup \tau^x(\Delta)}$ do not have the same \aggstable models, with~$x \in \{\mathit{cli},\dlv \}$.
\end{lemma}

\begin{proof}
    We proceed by cases.
    \emph{Case 1.} $\Gamma_1$ and $\Gamma_2$ do not have the same \aggmodels.
    Assume without loss of generality that~$I$ is an \aggmodel of~$\Gamma_1$ but not of~$\Gamma_2$.
    By Lemma~\ref{lem:strong.equivalence.only.if.cl}, it follows that~$I$ is an \aggstable model of~${\Gamma_1 \cup \tau^x(\Delta_I)}$ but not of~${\Gamma_2 \cup \tau^x(\Delta_I)}$.
    \emph{Case 2.} $\Gamma_1$ and $\Gamma_2$ have the same \aggmodels.
    By Lemma~\ref{lem:strong.equivalence.only.if.ht}, it follows that~$I$ is an \aggstable model of~$\Gamma_2 \cup \tau^x(\Delta_{\tuple{H,I}})$ but not of~$\Gamma_1 \cup \tau^x(\Delta_{\tuple{H,I}})$.
    In both cases,
    $\Gamma_2 \cup \tau^x(\Delta_{\tuple{H,I}})$
    and
    $\Gamma_1 \cup \tau^x(\Delta_{\tuple{H,I}})$
    have different \aggstable models.
\end{proof}

\begin{lemma}\label{lem:same.htmodels.same.stable.models.program}
    If~$\tau^x(\Pi_1)$ and $\tau^x(\Pi_2)$ have the same \agghtmodels, then~$\tau^x(\Pi_1 \cup \Delta)$ and $\tau^x(\Pi_2 \cup \Delta)$ they have the same \aggstable models, with~$x \in \{\mathit{cli},\dlv \}$, for any program~$\Delta$.
\end{lemma}

\begin{proof}
    Assume that~$\tau^x(\Pi_1)$ and $\tau^x(\Pi_2)$ have the same \agghtmodels.
    Then, $\tau^x(\Pi_1 \cup \Delta) = \tau^x(\Pi_1) \cup \tau^x(\Delta)$ has the same \agghtmodel as $\tau^x(\Pi_2 \cup \Delta) = \tau^x(\Pi_2) \cup \tau^x(\Delta)$.
    By Lemma~\ref{lem:same.htmodels.same.stable.models}, this implies that both have same \aggstable models.
\end{proof}

\begin{lemma}
    \label{lem:strong.equivalence.only.if.program}
    If~$\tau^x(\Pi_1)$ and $\tau^x(\Pi_2)$ do not have the same \agghtmodels, then there is some program~$\Delta$ without aggregates nor double negation such that~$\tau^x(\Pi_1 \cup \Delta)$ and $\tau^x(\Pi_2 \cup \Delta)$ do not have the same \aggstable models, with~$x \in \{\mathit{cli},\dlv \}$.
\end{lemma}

\begin{proof}
    In this case, by Lemma~\ref{lem:strong.equivalence.only.if} with~$\Gamma_1 = \tau^x(\Pi_1)$ and~$\Gamma_2 = \tau^x(\Pi_2)$ by noting that~$\tau^x(\Pi_1 \cup \Delta) = \tau^x(\Pi_1) \cup \tau^x(\Delta)$ and~$\tau^x(\Pi_2 \cup \Delta) = \tau^x(\Pi_2) \cup \tau^x(\Delta)$.
\end{proof}

\begin{proof}[Proof of Theorem~\ref{thm:strong.equivalece.clingo}]
    Assume that~$\taug(\Pi_1)$ and $\taug(\Pi_2)$ have the same \agghtmodels and let~$\Delta$ be a program.
    By Lemma~\ref{lem:same.htmodels.same.stable.models.program}, it follows that~$\taug(\Pi_1 \cup \Delta)$ and~$\taug(\Pi_2 \cup \Delta)$ have the same \aggstable models.
    This implies that~$\Pi_1 \cup \Delta$ and~$\Pi_2 \cup \Delta$ have the same fo\nobreakdash-clingo answer sets.
    By Theorem~\ref{thm:strong.equivalece.clingo}, this implies that~$\Pi_1 \cup \Delta$ and~$\Pi_2 \cup \Delta$ have the same clingo answer sets and, thus, are strongly equivalent under the \clingo\ semantics.
    \\[5pt]
    Conversely, suppose~$\taug(\Pi_1)$ and $\taug(\Pi_2)$ do not have the same \agghtmodels.
    By Lemma~\ref{lem:strong.equivalence.only.if.program}, there is a program~$\Delta$ such that~$\taug(\Pi_1 \cup \Delta)$ and~$\taug(\Pi_2 \cup \Delta)$ do not have the same \aggstable models, and, thus they do not have the same fo\nobreakdash-clingo answer sets.
    By Theorem~\ref{thm:abstrac.gringo.correspondence}, this implies~$\Pi_1 \cup \Delta$ and~$\Pi_2 \cup \Delta$ have different clingo answer sets.
    Hence, $\Pi_1$ and~$\Pi_2$ are not strongly equivalent under the \clingo\ semantics.
\end{proof}

\begin{proof}[Proof of Theorem~\ref{thm:strong.equivalece.dlv}]
    The proof is analogous to the one of Theorem~\ref{thm:strong.equivalece.clingo}.
    Assume that~$\taud(\Pi_1)$ and $\taud(\Pi_2)$ have the same \agghtmodels and let~$\Delta$ be a program.
    By Lemma~\ref{lem:same.htmodels.same.stable.models.program}, it follows that~$\taud(\Pi_1 \cup \Delta)$ and~$\taud(\Pi_2 \cup \Delta)$ have the same \aggstable models.
    This implies that~$\Pi_1 \cup \Delta$ and~$\Pi_2 \cup \Delta$ have the same fo\nobreakdash-dlv answer sets.
    By Theorem~\ref{thm:strong.equivalece.dlv}, this implies that~$\Pi_1 \cup \Delta$ and~$\Pi_2 \cup \Delta$ have the same dlv answer sets and, thus, are strongly equivalent under the \dlv\ semantics.
    \\[5pt]
    Conversely, suppose~$\taud(\Pi_1)$ and $\taud(\Pi_2)$ do not have the same \agghtmodels.
    By Lemma~\ref{lem:strong.equivalence.only.if.program}, there is a program~$\Delta$ such that~$\taud(\Pi_1 \cup \Delta)$ and~$\taud(\Pi_2 \cup \Delta)$ do not have the same \aggstable models, and, thus they do not have the same fo\nobreakdash-dlv answer sets.
    By Theorem~\ref{thm:dlv.correspondence}, this implies~${\Pi_1 \cup \Delta}$ and~${\Pi_2 \cup \Delta}$ have different dlv answer sets.
    Hence, $\Pi_1$ and~$\Pi_2$ are not strongly equivalent under the \dlv\ semantics.
\end{proof}

\begin{proof}[Proof of Theorem~\ref{thm:strong.equivalece.croos}]
    The proof is similar to the those of Theorems~\ref{thm:strong.equivalece.clingo} and~\ref{thm:strong.equivalece.dlv}.
    Assume that~$\taug(\Pi_1)$ and $\taud(\Pi_2)$ have the same \agghtmodels and let~$\Delta$ be a program such that~$\taug(\Delta)$ and $\taud(\Delta)$ have the same \agghtmodels.
    Then, ${\taug(\Pi_1 \cup \Delta)} = {\taug(\Pi_1) \cup \taug(\Delta)}$ and~${\taud(\Pi_2) \cup \Delta} = {\taud(\Pi_2) \cup \taud(\Delta)}$ have the same \agghtmodels.
    By Lemma~\ref{lem:same.htmodels.same.stable.models}, it follows that~$\taud(\Pi_1 \cup \Delta)$ and~$\taud(\Pi_2 \cup \Delta)$ have the same \aggstable models.
    This implies that the fo\nobreakdash-clingo answer sets of~$\Pi_1 \cup \Delta$ and the same fo\nobreakdash-dlv answer sets of~$\Pi_2 \cup \Delta$ coincide.
    By Theorems~\ref{thm:strong.equivalece.clingo} and~\ref{thm:strong.equivalece.dlv}, this implies that clingo answer sets~$\Pi_1 \cup \Delta$ coincide with the dlv answer sets of~$\Pi_2 \cup \Delta$.
    \\[5pt]
    Conversely, suppose~$\taug(\Pi_1)$ and $\taud(\Pi_2)$ do not have the same \agghtmodels.
    By Lemma~\ref{lem:strong.equivalence.only.if}, there is a program~$\Delta$ without aggregates nor double negation such that ${\taug(\Pi_1 \cup \Delta)} = {\taug(\Pi_1) \cup \tau^x(\Delta)}$ and ${\taud(\Pi_2 \cup \Delta)} = {\taud(\Pi_2) \cup \tau^x(\Delta)}$ do not have the same \aggstable models.
    Note that, since~$\Delta$ does not contain aggregates nor double negation, it follows that~$\taug(\Delta) = \taud(\Delta)$.
    Hence, the fo\nobreakdash-clingo answer sets of ${\Pi_1 \cup \Delta}$ do not coincide with the fo\nobreakdash-dlv answer sets of~${\Pi_2 \cup \Delta}$.
    By Theorems~\ref{thm:strong.equivalece.clingo} and~\ref{thm:strong.equivalece.dlv}, this implies that the clingo answer sets of~${\Pi_1 \cup \Delta}$ are different from the dlv answer sets of~${\Pi_2 \cup \Delta}$.
\end{proof}
\subsection{Proof of Section Strong Equivalence using Classical Logic}

\begin{lemma}\label{lem:prima.transformation.term}
    $\hat{t}^{I^H} = t^I$.
\end{lemma}

\begin{proof}
    If~$t$ is of the form~$f()$ where~$f$ is an extensional function, then~${\hat{t} = f()}$ and the result follows by definition.
    If~$t$ is of the form~$f()$ where~$f$ is an intensional function, then~${\hat{t} =\hat{f}()}$ and~${\hat{f}^{I^H} = f^I}$ follows by definition.
    The rest of the proof follows by induction on the structure of~$t$ in a similar way.
\end{proof}

\begin{lemma}\label{lem:prima.transformation.atom}
    $I^H\models \hat{p}(\hat{\boldt})\hbox{ iff }I\models p(\boldt)$.
\end{lemma}

\begin{proof}
    By Lemma~\ref{lem:prima.transformation.term}, we get~${\hat{\boldt}^{I^H} = \boldt^I}$.
    Then, $I^H\models \hat{p}(\hat{\boldt})$ iff ${I^H\models \hat{p}(\hat{\boldt})}$ iff~${I^H\models \hat{p}((\hat{\boldt}^{I^H})^*)}$ iff $I^H\models \hat{p}((\hat{\boldt}^{I})^*)$ iff $I\models p((\hat{\boldt}^{I})^*)$ iff $I\models p(\boldt)$.
\end{proof}

\begin{lemma}\label{lem:prima.transformation.cl}
    ${I^H\models \hat{F}}$ iff ${I \models F}$.
\end{lemma}

\begin{proof}
    We will consider the case of a ground
    atom~$A$ of the form~$p(\boldt)$; extension to arbitrary sentences by induction is
    straightforward.
    If $p$ is extensional then $\hat{A}$ is~$p(\hat{\boldt})$;
    $I^H\models p(\hat{\boldt})$ iff $I\models p(\boldt)$ follows by Lemma~\ref{lem:prima.transformation.term} and the fact that~$I^H$ interprets extensional symbols in the same way as~$I$.
    If~$p$ is intensional then $\hat{A}$ is~$\hat{p}(\hat{\boldt})$,
    and the result follows by Lemma~\ref{lem:prima.transformation.atom}.
  \end{proof}

\begin{lemma}\label{lem:nonprima.transformation.term}
    Let~$t$ is a term of~$\sigma$.
    Then, $t^{I^H} = t^H$.
\end{lemma}

\begin{proof}
    If~$t$ is of the form~$f()$ where~$f$ is an extensional function, then~${t = f()}$ and we have~${f^{I^H} = f^I = f^H}$.
    If~$t$ is of the form~$f()$ where~$f$ is an intensional function, then~${t = f()}$ and~${f^{I^H} = f^H}$ follows by definition.
    The rest of the proof follows by induction on the structure of~$t$ in a similar way.
\end{proof}

\begin{lemma}\label{lem:nonprima.transformation}
    Let~$F$ be a formula of~$\sigma$.
    Then, ${I^H\models \NN{F}}$ iff ${H \models F}$.
\end{lemma}

\begin{proof}
    The proof is by induction on the number of connectives and quantifiers in~$F$.
    We consider below the more difficult cases when~$F$ is an atomic formula or the negation~$\sneg$.

    \medskip\noindent\emph{Case~1:}
    $F$ is an atomic formula~$p(\boldt)$.
    Let~$J = I^H$.
    Then, $\NN{F}$ is also~$p(\boldt)$.
    By Lemma~\ref{lem:nonprima.transformation.term}, we have~${\boldt^{J} = \boldt^H}$.
    Let~$\boldd$ be the common value of $\boldt^{I^H}$ and $\boldt^H$.
    \emph{Case~1.1:}~$p$ is intensional.
    The left-hand side is equivalent to
    ${I^H\models p(\boldd^*)}$ and consequently to ${H \models p(\boldd^*)}$.
    The right-hand side to is equivalent~$H \models p(\boldd^*)$ as well.
    \emph{Case~1.2:}~$p$ is extensional.
    Each of two sides is equivalent
    to~$I\models p(\boldd^*)$.

    \medskip\noindent\emph{Case~2:} $F$ is $\sneg G$.
    Then, $\NN{F}$ is~$\neg \NN{G}$; we need to check that~$I^H \models \neg\NN{G}$ iff~$H \models \neg G$.
    This is equivalent to check~$I^H \not\models \NN{G}$ iff~$H \not\models G$, which follows by induction hypothesis.
\end{proof}

\begin{proof}[Proof of Proposition~\ref{prop:prima.transformation.ht}]
    We prove it for a formula~$F$ and its extension to theories is straightforward.
    The proof is by induction on the number of propositional connectives and quantifiers in~$F$.
    We consider below the more difficult cases when~$F$ is an atomic formula, one of the negations, or an implication.

    \medskip\noindent\emph{Case~1:}
    $F$ is an atomic formula~$p(\boldt)$.
    Then~$\gamma F$ is~$F \wedge \hat{F}$; we need to check that
    \begin{align*}
        I^H\models F \wedge \hat{F} \hbox{ iff }\tuple{H,I}\modelsht F.
    \end{align*}
    On the one hand, $\tuple{H,I} \modelsht F$
    iff (by definition)~${I \models F}$ and~$H \models F$
    iff (Lemma~\ref{lem:prima.transformation.cl})~${I^H \models F}$ and~${I \models F}$.
    On the other hand, ${I^H\models F \wedge \hat{F}}$ iff ${I^H\models F}$ and ${I^H\models \hat{F}}$ iff ${I\models F}$ and ${I^H\models F}$.
    Hence, it is enough to show that
    \begin{align*}
        I^H\models F \hbox{ iff } H \models F.
    \end{align*}
    which follows by Lemma~\ref{lem:nonprima.transformation}.

    \medskip\noindent\emph{Case~2:} $F$ is $\neg G$.
    Then $\gamma F$ is
    $\neg \hat{G}$; we need to check that
    $$I^H\not\models \hat{G}\hbox{ iff }\tuple{H,I}\modelsht\neg G.$$
    By Lemma~\ref{lem:prima.transformation.cl}, the left-hand side is
    equivalent to $I\not\models G$.
    By Proposition~\ref{prop:persistence},
    the right-hand side is equivalent to $I \not\models G$ as well.

    \medskip\noindent\emph{Case~3:} $F$ is $\sneg G$.
    Then $\gamma F$ is
    $\neg \NN{G} \wedge \neg \hat{G}$; we need to check that
    $$I^H\models \neg \NN{G} \wedge \neg \hat{G}\hbox{ iff }\tuple{H,I}\modelsht\sneg G.$$
    The left-hand side is
    equivalent to the conjunction of~${I^H \not\models \NN{G}}$ and~${I^H \not\models \hat{G}}$.
    By Lemmas~\ref{lem:nonprima.transformation} and~\ref{lem:prima.transformation.cl}, this is equivalent to the conjunction of~${H \not\models G}$ and~${I \not\models G}$, which by definition, is equivalent to the right-hand side.

    \medskip\noindent\emph{Case~4:} $F$ is of the form~${F_1 \to F_2}$.
    Then $\gamma F$
    is the conjunction $(\gamma F_1 \to \gamma F_2) \land (\hat{F_1} \to \hat{F_2})$, so that the condition
    $I^H\models \gamma F$ holds iff
    \begin{align}
    I^H\not\models \gamma F_1\hbox{ or }I^H\models \gamma F_2
    \label{eq:d1:prop:prima.transformation.ht}
    \end{align}
    and
    \begin{align}
    I^H\models \hat{F_1}\to \hat{F_2}.
    \label{eq:d2:prop:prima.transformation.ht}
    \end{align}
    By the induction hypothesis,~(\ref{eq:d1:prop:prima.transformation.ht}) is equivalent to
    \begin{align}
    \tuple{H,I}\not\modelsht F_1 \hbox{ or }\tuple{H,I}\modelsht F_2.
    \label{eq:d3:prop:prima.transformation.ht}
    \end{align}
    By Lemma~\ref{lem:prima.transformation.cl}, we get that~(\ref{eq:d2:prop:prima.transformation.ht}) is equivalent to
    \begin{align}
    I\models G\to H.
    \label{d4}
    \end{align}
    By definition, the conjunction of~(\ref{eq:d3:prop:prima.transformation.ht}) and~(\ref{d4}) is equivalent to
    $\tuple{H,I}\modelsht F_1 \to F_2$.
\end{proof}

\begin{lemma}\label{lem:prima.interpretation}
    An interpretation of the signature~$\hat\sigma$ satisfies~$\HT$ iff it can be
    represented in the form $I^H$ for some \htinterp~$\tuple{H,I}$.
\end{lemma}

\begin{proof}
    For the if-part, take any sentence of the form of~$\forall\boldX (p(\boldX) \to \hat{p}(\boldX))$ from~$\HT$.
    We need to show that $I^H$ satisfies all sentences of the form
    ${p(\boldd^*)\to \hat{p}(\boldd^*)}$.
    Assume that ${I^H\models p(\boldd^*)}$.
    Then, ${\boldd^* \in p^H \subseteq p^I}$,
    and consequently ${I\models p(\boldd^*)}$, which is equivalent to
    $I^H\models \hat{p}(\boldd^*)$.

    For the only-if part, take any
    interpretation~$J$ of~$\hat\sigma$ that satisfies~$\HT$.  Let~$I$ be the
    interpretation of~$\sigma$ that has the same domains as~$J$,
    interprets extensional symbols in the same way as~$J$, interprets every intensional function~$f$ in accordance with the condition~$f^I=\hat{f}^J$. and interprets every intensional~$p$ in
    accordance with the condition
    \begin{gather}
        I\models p({\bf d}^*)\hbox{ iff }J\models \hat{p}({\bf d}^*).
        \label{iff}
    \end{gather}
    Similarly, let~$H$ be the interpretation of~$\sigma$ that has the same domains as~$J$, interprets extensional symbols in the same way as~$J$, interprets every intensional function~$f$ in accordance with the condition~${f^H=f^J}$, and interprets every intensional~$p$ in
    accordance with the condition
    \begin{gather}
        H\models p({\bf d}^*)\hbox{ iff }J\models p({\bf d}^*).
        \label{iff2}
    \end{gather}
    Then, $I$ and~$H$ agree on all extensional symbols.
    Furthermore, since~$J$ satisfies~$\AGG$, $J$ satisfies~$\hat{p}(\boldd^*)$ for every atom~$p(\boldd^*)$ satisfied by~$J$.
    By~\eqref{iff} and~\eqref{iff2}, it follows that all atoms satisfied by~$H$ are satisfied by~$I$.
    It follows that $\tuple{H,I}$ is an \htinterp.
    Let us show that ${I^H=J}$.
    Each of the interpretations~$I^H$ and~$J$ has the same domains as~$I$ and interprets all extensional symbols in the same way as~$I$.
    Furthermore, for intensional function symbols, we have
    \begin{align*}
        \hat{f}^{I^H} &= f^I = \hat{f}^J
        \\
        f^{I^H} &= f^H = f^J.
    \end{align*}
    For everyintensional~$p$ and any tuple~$\boldd$ of elements of appropriate
    domains, each of the conditions $I^H\models p(\boldd^*)$,
    $J\models p(\boldd^*)$ is equivalent to $H \models p(\boldd^*)$,
    and each of the conditions $I^H\models \hat{p}(\boldd^*)$,
    $J\models \hat{p}(\boldd^*)$ is equivalent to $I\models p(\boldd^*)$.
\end{proof}

\begin{lemma}\label{lem:agg_element.aux2}
    Let~$E$ be an aggregate element of the form~\eqref{eq:agel:2} with global variables~$\boldX$ and local variables~$\boldY$ and let~$\tuple{H,I}$ be a standard \htinterp.
    Let~$\boldx$ be a list of ground terms of sort~$\sortsuper$ of the same length as~$\boldX$,
    $\dtuple$ a domain element of sort tuple,
    $l_i' = (l_i)^\boldX_{\boldx}$
    and $t_i' = (t_i)^\boldX_{\boldx}$.
    Then, $I^H$ satisfies
    \begin{gather}
        \pmemF{\dtuple^*}{\hat{\sets}^x_{|E|}(\boldx)}\leftrightarrow
        \exists \boldY\, \hat{F}
    \label{eq:1:lem:agg_element.aux2}
    \end{gather}
    where~$F$ is~$(\dtuple^* = \ftuple(t_1',\dots,t_m') \wedge l_1'\wedge\cdots\wedge l_n')$
    iff the following two conditions are equivalent:
    \begin{enumerate}
        \item
        \label{item:1:lem:agg_element.aux2}
        $\dtuple$ belongs to~$\hat{\sets}^x_{|E|}(\boldx)^I$,
        \item
        \label{item:2:lem:agg_element.aux2}
        there is a list~$\boldc$ of domain elements of sort~$\sortsuper$ of the same length as~$\boldY$ such that
        $\dtuple \ = \ \tuple{(t_1'')^I,\dots,(t_m'')^I}$ and
        $I$ satisfies~$l_1'' \wedge \dots \wedge l_n''$
        with $t_i'' = (t_i')^\boldY_{\boldy}$ and~$l_i'' = (l_i')^\boldY_{\boldy}$ and~$\boldy = \boldc^*$.
    \end{enumerate}
\end{lemma}

\begin{proof}
    $I$ satisfies~\eqref{eq:1:lem:agg_element.aux2} iff condition~\ref{item:1:lem:agg_element.aux2} is equivalent to:
    \begin{itemize}
        \item[]
        \label{item:5:lem:agg_element.aux2}
        $I$ satisfies $\exists \boldY\, (\dtuple^* = \ftuple(t_1',\dots,t_m') \wedge \hat{l_1}'\wedge\cdots\wedge \hat{l_n}')$
    \end{itemize}
    Furthermore, the latter holds iff there is a list~$\boldc$ of domain elements of sort~$\sortsuper$ such that
    $$
    \domaintuple \ = \ \ftuple(t_1'',\dots,t_m'')^I \ = \ \tuple{(t_1'')^I,\dots,(t_m'')^I}
    $$
    and $I$ satisfies $l_1''\wedge\cdots\wedge l_n''$ with~$\boldy = \boldc^*$ iff condition~\ref{item:2:lem:agg_element.aux2} holds.
\end{proof}

\begin{lemma}
    \label{lem:prima.agg-interp.cl}
    Let~$\tuple{H,I}$ be a standard \htinterp.
    Then, $I$ is an \agginterp iff~$I^H$ satisfies~\eqref{eq:agg_element.clingo.there} for every~$E/\boldX$ in~$\S$.
\end{lemma}

\begin{proof}
    By Lemma~\ref{lem:prima.transformation.cl}, it follows that~$I^H$ satisfies~\eqref{eq:agg_element.clingo.there} iff~$I$ satisfies
    ${\forall \boldX \vartuple \big(\pmemF{\vartuple}{\sets^x_{|E/\boldX|}(\boldX)}
                \leftrightarrow
                \exists \boldY F^{\mathit{x}}\big)}$.
    Then, $I$ satisfies
    \begin{align}
        \forall \vartuple \big(\pmemF{\vartuple}{\sets^x_{|E/\boldX|}(\boldx)}
                &\leftrightarrow
                \exists \boldY (F^{\mathit{x}})^\boldX_\boldx\big)
        \label{eq:agg_element.clingo.there.noprima}
    \end{align}
        iff~$\sets_{|E/\boldX|}^x(\boldx)^I$ is the set of all tuples~$\dtuple$ such that, there is~$\boldy$ such that~$I$ satisfies
    \begin{align}
        \dtuple^* = \ftuple(t_1',\dots,t_k') \wedge  \tau^x(l_1')\wedge\cdots\wedge \tau^x(l_m')
    \end{align}
    with~$t_i' = (t_i)^{\boldX\boldY}_{\boldx\boldy}$ and~$l_i' = (l_i)^{\boldX\boldY}_{\boldx\boldy}$
    iff~$\sets_{|E/\boldX|}^x(\boldx)^I$ is the set of all tuples of the form
    \begin{align*}
        \tuple{t_1',\dots,t_k'}
    \end{align*}
    such that $I$ satisfies
    \begin{align*}
        \tau^x(l_1') \wedge \dots \wedge \tau^x(l_m')
    \end{align*}
    and the result holds.
\end{proof}

\begin{lemma}
    \label{lem:prima.agg-interp.clingo}
    Let~$\tuple{H,I}$ be a standard \htinterp such that~$I$ is an \agginterp.
    Then, $\tuple{H,I}$ satisfies the condition of being an \agginterp for every term of the form~$\setsg_{|E/\boldX|}(\boldx)$ iff~$I^H$ satisfies~\eqref{eq:agg_element.clingo.here} for every~$E/\boldX$ in~$\S$.
\end{lemma}

\begin{proof}
    $I^H$ satisfies~\eqref{eq:agg_element.clingo.here} iff the following two conditions are equivalent for every tuple~$\boldd$ of domain element of sort general and every domain element~$\dtuple$ of sort tuple:
    \begin{itemize}
        \item $I^H$ satisfies~$\pmemF{\dtuple^*}{\setsg_{|E/\boldX|}(\boldx)}$, and
        \item $I^H$ satisfies~$\exists \boldY \gamma (F^{\mathit{cli}})$.
    \end{itemize}
    By Lemma~\ref{lem:nonprima.transformation} and Proposition~\ref{prop:prima.transformation.ht} respectively, these two conditiosn are equivalent to
    \begin{itemize}
        \item $H$ satisfies~$\pmemF{\dtuple^*}{\setsg_{|E/\boldX|}(\boldx)}$, and
        \item $\tuple{H,I}$ satisfies~$\exists \boldY F^{\mathit{cli}}$.
    \end{itemize}
    The rest of the proof is analogous to that of Lemma~\ref{lem:prima.agg-interp.cl}.%
\end{proof}

\begin{lemma}
    \label{lem:prima.agg-interp.dlv}
    Let~$\tuple{H,I}$ be a standard \htinterp such that~$I$ is an \agginterp.
    Then, $\tuple{H,I}$ satisfies the condition of being an \agginterp for every term of the form~$\setsd_{|E/\boldX|}(\boldx)$ iff~$I^H$ satisfies~\eqref{eq:agg_element.dlv.here} for every~$E/\boldX$ in~$\S$.
\end{lemma}

\begin{proof}
    $I^H$ satisfies~\eqref{eq:agg_element.dlv.here} iff the following two conditions are equivalent for every tuple~$\boldd$ of domain element of sort general and every domain element~$\dtuple$ of sort tuple:
    \begin{itemize}
        \item $I^H$ satisfies~$\pmemF{\dtuple^*}{\setsg_{|E/\boldX|}(\boldx)}$, and
        \item $I^H$ satisfies~$\exists \boldY \NN{F^{\mathit{cli}}}$.
    \end{itemize}
    By Lemma~\ref{lem:nonprima.transformation}, these two conditiosn are equivalent to
    \begin{itemize}
        \item $H$ satisfies~$\pmemF{\dtuple^*}{\setsg_{|E/\boldX|}(\boldx)}$, and
        \item $H$ satisfies~$\exists \boldY F^{\mathit{dlv}}$.
    \end{itemize}
    The rest of the proof is analogous to that of Lemma~\ref{lem:prima.agg-interp.cl}.%
\end{proof}

\begin{lemma}
    \label{lem:prima.agg-interp}
    Every standard \htinterp~$\tuple{H,I}$ is an \agginterp iff~$I^H$ satisfies~$\AGG$.
\end{lemma}

\begin{proof}
    Directly by Lemmas~\ref{lem:prima.agg-interp.cl}, \ref{lem:prima.agg-interp.clingo} and~\ref{lem:prima.agg-interp.dlv}.
\end{proof}

\begin{proof}[Proof of Proposition~\ref{prop:prima.agg.interpretation}]
    For the if-part, take any interpretation~$J$ of~$\hat\sigma$ that satisfies~$\HT$ and~$\AGG$.
    Since~$J$ satisfies~$\HT$, by Lemma~\ref{lem:prima.interpretation}, there is an \htinterp~$\tuple{H,I}$ such that~${I^H=J}$.
    Furthermore, since~$J$ satisfies~$\AGG$, by Proposition~\ref{lem:prima.agg-interp}, $\tuple{H,I}$ is an \agginterp.
    For the only-if part, take any \agghtinterp~$\tuple{H,I}$.
    By Proposition~\ref{lem:prima.agg-interp}, $I^H$ satisfies~$\AGG$.
    By Lemma~\ref{lem:prima.interpretation}, $I^H$ satisfies~$\HT$.
\end{proof}

\begin{proof}[Proof of Theorem~\ref{thm:strong.equivalece.classical}]
    First note that, since~$\Pi_1$ and~$\Pi_2$ are finite, there is a finite number of sentences in~$\HT$, $\AGG$, $\gamma\tau^x\Pi_1$, and~$\gamma\tau^x\Pi_2$.
    Hence, the conjunction of all the sentences in them is well\nobreakdash-defined.
    Let~$G_i$ be the conjunction of all sentences in~$\tau^x\Pi_i$, so that~$F_i$ is~$\gamma G_i$.
    We want to show that $\Pi_1$ is strongly equivalent to~$\Pi_2$ iff
    \begin{align*}
        &\text{for every standard model~$J$ of~$HT\cup\AGG$}
        \\
        &J\models \gamma G_1\hbox{ iff }J\models \gamma G_2.
    \end{align*}
    By Proposition~\ref{prop:prima.agg.interpretation}, this is further equivalent to the condition
    \begin{align*}
        &\text{for every \agghtinterp~$\tuple{H,I}$}
        \\
        &I^H\models \gamma G_1\hbox{ iff }I^H\models \gamma G_2.
    \end{align*}
    By Proposition~\ref{prop:prima.transformation.ht}, this is further equivalent to stating that
    \begin{align*}
        &\text{for every \agghtinterp~$\tuple{H,I}$}
        \\
        &\tuple{H,I}\modelsht G_1\hbox{ iff }\tuple{H,I}\modelsht G_2.
    \end{align*}
    that is, to the condition
    $$
    \hbox{$\tau^x\Pi_1$ and $\tau^x\Pi_2$ have the same \agghtmodels.}
    $$
    By Theorems~\ref{thm:strong.equivalece.clingo} and~\ref{thm:strong.equivalece.dlv}, this is equivalent to the condition that $\Pi_1$ and~$\Pi_2$ are strongly equivalent under the \clingo and \dlv semantics, respectively.
\end{proof}

\end{document}